\documentclass{article}

\usepackage{times}
\usepackage{graphicx} 
\usepackage{subfigure} 

\usepackage{natbib}

\usepackage{algorithm}
\usepackage{algorithmic}

\usepackage{hyperref}

\usepackage[accepted]{icml2014}

\usepackage{enumerate}
\usepackage{enumitem}
\usepackage{bm}
\usepackage{amssymb}
\usepackage{amsthm}
\usepackage{amsmath}
\usepackage{graphicx}
\usepackage{mathtools}
\usepackage[stable]{footmisc}
\usepackage{floatflt}
\usepackage{wrapfig}
\allowdisplaybreaks

\newtheorem{theorem}{Theorem}

\newtheorem{lemma}{Lemma}
\newtheorem{claim}{Claim}
\newtheorem{property}{Property}
\newtheorem{proposition}{Proposition}
\newtheorem{assumption}{Assumption}
\newtheorem{example}{Example}

\newcommand{\LS}{{\bf LS}}
\newcommand{\Z}{{\bf Z}}
\newcommand{\M}{{\bf M}}
\newcommand{\x}{{\bf x}}
\newcommand{\w}{{\bf w}}
\newcommand{\W}{{\bf W}}
\newcommand{\F}{\mathcal{F}}
\newcommand{\MM}{\mathcal{M}}
\renewcommand{\P}{{\bf P}}
\newcommand{\e}{{\bf e}}
\newcommand{\Reg}{{\bf Reg}}
\newcommand{\Alg}{{\bf Alg}}
\newcommand{\0}{{\bf 0}}
\newcommand{\V}{\tilde{V}}

\newcommand{\BV}{\bar{V}}
\newcommand{\ra}{\rightarrow}
\newcommand{\E}{\mathbb{E}}
\newcommand{\R}{\mathbb{R}}

\newcommand{\bxi}{{\bm\xi}}
\newcommand{\1}{{\bm 1}}

\renewcommand{\(}{\left(}
\renewcommand{\)}{\right)}

\icmltitlerunning{Towards Minimax Online Learning with Unknown Time Horizon}

\begin{document} 

\twocolumn[
\icmltitle{Towards Minimax Online Learning with Unknown Time Horizon}

\icmlauthor{Haipeng Luo}{haipengl@cs.princeton.edu}
\icmladdress{Department of Computer Science, Princeton University, Princeton, NJ 08540}
\icmlauthor{Robert E. Schapire}{schapire@cs.princeton.edu}
\icmladdress{Department of Computer Science, Princeton University, Princeton, NJ 08540}

\icmlkeywords{online learning, unknown horizon, random horizon, minimax analysis}

\vskip 0.3in
]

\begin{abstract} 
We consider online learning when the time horizon is unknown.
We apply a minimax analysis, beginning with the fixed horizon case,
and then moving on to two unknown-horizon settings, one that assumes
the horizon is chosen randomly according to some known distribution,
and the other which allows the adversary full control over the horizon.
For the random horizon setting with restricted losses, 
we derive a fully optimal minimax algorithm.
And for the adversarial horizon setting, we prove a nontrivial lower
bound which shows that the adversary obtains strictly more power
than when the horizon is fixed and known.
Based on the minimax solution of the random horizon setting, 
we then propose a new adaptive algorithm which
``pretends'' that the horizon is drawn from a distribution from a
special family, but no matter how the actual horizon is chosen, 
the {\it worst-case} regret is of the optimal rate.
Furthermore, our algorithm can be combined and applied in
many ways, for instance, to online convex optimization, 
follow the perturbed leader, exponential weights algorithm
and first order bounds.
Experiments show that our algorithm outperforms many other existing algorithms
in an online linear optimization setting.

\end{abstract} 

\section{Introduction}
We study online learning problems with unknown time horizon with the
aim of developing algorithms and approaches for the realistic case
that the number of time steps is initially unknown.

We first adopt the standard Hedge setting \cite{FreundSc97}
where the learner chooses a distribution over $N$ actions on each
round, and the losses for each action are then selected by an
adversary.  
The learner incurs loss equal to the expected loss of the actions in
terms of the distribution it chose for this round, and its goal is to
minimize the regret, the difference between its cumulative loss and
that of the best action after $T$ rounds.

Various algorithms are known to achieve the optimal (up to a constant)
upper bound $O(\sqrt{T\ln N})$ on the regret.
Most of them assume that the horizon $T$ is known ahead of time,
especially those which are minimax optimal.
When the horizon is unknown, the so-called doubling
trick~\cite{CesabianchiFrHeHaScWa97}
is a general technique to make a learning algorithm adaptive and still
achieve
$O(\sqrt{T\ln N})$ regret uniformly for any $T$. 
The idea is to first guess a horizon, 
and once the actual horizon exceeds this guess,
double it and restart the algorithm.
Although, in theory, it is widely applicable,
the doubling trick is aesthetically inelegant, and intuitively
wasteful, since it repeatedly restarts itself, entirely forgetting all
the preceding information.
Other approaches have also been proposed, as we discuss shortly.

In this paper, we study the problem of learning with unknown horizon
in a game-theoretic framework.
We consider a number of variants of the problem, and make progress
toward a minimax solution.
Based on this approach, we give a new general technique which can also
make other minimax or non-minimax algorithms adaptive and achieve low regret
in a very general online learning setting.
The resulting algorithm is still not exactly optimal, but it makes use
of all the previous information on each round and achieves much lower
regret in experiments.

We view the Hedge problem as a repeated game between the
learner and the adversary. \citet{AbernethyWaYe08}, and
\citet{AbernethyWa10} proposed an exact minimax optimal solution for a
slightly different game with binary losses, assuming that the loss of
the best action is at most some fixed constant.
They derived the solution under a very simple type of loss space; that is, on
each round only one action suffers one unit loss. 
We call this the basis vector loss space. 
As a preliminary of this paper, 
we also derive a similar minimax solution under this simple loss space for our
setting where the horizon $T$ is fixed and known to the learner ahead of time. 

We then move on to the primary interest of this paper, 
that is, the case when the horizon is unknown to the learner. 
We study this unknown horizon setting in the minimax framework, 
with the aim of ultimately deriving game-theoretically optimal algorithms. 
Two types of models are studied. 
The first one assumes the horizon is chosen according to some known distribution, 
and the learner's goal is to minimize the expected regret. 
We show the exact minimax solution for the basis vector loss space in this case. 
It turns out that the distribution the learner should choose on each round is simply the conditional expectation of the distributions the learner would have chosen for the fixed horizon case. 

The second model we study gives the adversary the power to decide the horizon on the fly, which is possibly the most adversarial case. 
In this case, we no longer use the regret as the performance measure. 
Otherwise the adversary would obviously choose an infinite horizon. 
Instead, we use a scaled regret to measure the performance. 
Specifically, we scale the regret at time $t$ by the optimal regret under fixed horizon $t$. 
The exact optimal solution in this case is unfortunately not found and remains an open problem, even for the extremely simple case. 
However, we give a lower bound for this setting to show that the optimal regret is strictly greater than the one in the fixed horizon game. 
That is, the adversary does obtain strictly more power if allowed to pick the horizon. 

We then propose our new adaptive algorithm based on the minimax solution in the random horizon setting. 
One might doubt how realistic a random horizon is in practice. 
Even if the true horizon is indeed drawn from a fixed distribution, 
how can we know this distribution? 
We address these problems at the same time.
Specifically, we prove that no matter how the horizon is chosen, 
if we assume it is drawn from a distribution from a special family, 
and let the learner play in a way similar to the one in the random horizon setting, 
then the {\it worst-case} regret at any time $T$ (not the expected regret) can still be of the optimal order.
In other words, although the learner is behaving as if the horizon is random, 
its regret will be small even if the horizon is actually controlled by an adversary.
Moreover, the results hold for not just the Hedge problem, but
a general online learning setting that includes many interesting problems.

Our idea can be combined not only with the minimax algorithm, 
but also the ``follow the perturbed leader'' algorithm and the exponential weights algorithm. 
In addition, our technique can not only deal with unknown horizon,
but also other unknown information such as the loss of the best action,
thus leading to a first order regret bound that depends on the loss of the best action
\cite{CesabianchiLu06}.
Like the doubling trick, this seems to be a quite general way to make an algorithm adaptive.
Furthermore, we conduct experiments showing that our algorithm outperforms 
many existing algorithms, including the doubling trick, 
in an online linear optimization setting within an $\ell_2$ ball
where our algorithm has an explicit closed form.


The rest of the paper is organized as follows. 
We define the Hedge setting formally in Section \ref{sec:def}, 
and derive the minimax solution for the fixed horizon setting as the preliminary of this paper in Section \ref{sec:fh}. 
In Section \ref{sec:uh}, we study two unknown horizon settings in the minimax framework. 
We then turn to a general online learning setting and present
our new adaptive algorithm in Section \ref{sec:alg}.
Implementation issues,  experiments, 
and applications are discussed in Section \ref{sec:gen}. 
We omit most of the proofs due to space limitations,
but all details can be found in the supplementary material.

\paragraph{Related work}
Besides the doubling trick, other adaptive algorithms have been studied \cite{AuerCeGe02, Gentile03, YaroshinskyElSe04, ChaudhuriFrHs09}. 
\citet{AuerCeGe02} showed that for algorithms such as the exponential weights algorithm \cite{LittlestoneWa94, FreundSc97, FreundSc99}, 
where a learning rate $\eta$ should be set as a function of the horizon, 
typically in the form $\sqrt{(b\ln N)/T}$ for some constant $b$, 
one can simply set the learning rate adaptively as $\sqrt{(b\ln N)/t}$, 
where $t$ is the current number of rounds. 
In other words, this algorithm always pretends the current round is the last round. 
Although this idea works with the exponential weights algorithm, 
we remark that assuming the current round is the last round does not always work. 
Specifically, one can show that it will fail if applied to the minimax algorithm (see Section \ref{subsec:exp_alg}). 
In another approach to online learning with unknown horizon, 
\citet{ChaudhuriFrHs09} proposed an adaptive algorithm based on a novel potential function reminiscent of the half-normal distribution.

Other performance measures different from the usual regret were studied before.
\citet{FosterVo98} introduced internal regret comparing the loss of an
online algorithm to the loss of a modified algorithm which consistently replaces one action by another.
\citet{HerbsterWa95}, and \citet{BousquetWa03} compared the learner's loss with the best $k$-shifting expert, while \citet{HazanSe07} studied the usual regret within any time interval. 
To the best of our knowledge, the form of scaled regret that we study is new. 
Lower bounds on anytime regret in terms of the quadratic variations for any loss sequence 
(instead of the worst case sequence this paper considers) were studied by \citet{GoferMa12}.

\section{Repeated Games}\label{sec:def}
We first consider the following repeated game between a learner and an adversary. 
The learner has access to $N$ actions.
On each round $t = 1, \ldots, T$,
(1) the learner chooses a distribution $\P_t$ over the actions;
(2) the adversary reveals the loss vector $\Z_t = (Z_{t,1}, \ldots, Z_{t, N}) \in \LS$, 
	where $Z_{t,i}$ is the loss for action $i$ for this round, 
	and the {\it loss space} $\LS$ is a subset of $[0, 1]^N$;
(3) the learner suffers loss $\ell_t = \P_t \cdot \Z_t$ for this round.

Notice that the adversary can choose the losses on round $t$ with full knowledge of the history $\P_{1:t} $ 
and $\Z_{1:t-1}$ 
,
that is, all the previous choices of the learner and the adversary
(we use notation $a_{1:t}$ to denote the multiset $\{a_1, \ldots, a_t\}$).
We also denote the cumulative loss up to round $t$ for the learner and the actions by  $L_t = \sum_{t'=1}^{t} \ell_{t'}$ and $\M_t = \sum_{t'=1}^{t} \Z_{t'}$ respectively. 
The goal for the learner is to minimize the difference between its total loss and that of the 
best action at the end of the game. 
In other words, the goal of the learner is to minimize $\Reg(L_T, \M_T)$,  
where we define the regret function $ \Reg(L, \M) \triangleq L - \min_i M_i , $ 
for $L\in \R$ and $\M\in \R^N$. 
The number of rounds $T$ is called the {\it horizon}. 

Regarding the loss space $\LS$, perhaps the simplest one is $\{\e_1, \ldots, \e_N\}$, 
the $N$ standard basis vectors in $N$ dimensions. 
Playing with this loss space means that on each round, 
the adversary chooses one single action to incur one unit loss. 
In order to show the intuition of our main results,  
we mainly focus on this basis vector loss space 
in Sections \ref{sec:fh} and \ref{sec:uh}, 
but we return to the most general case $[0,1]^N$ later.

\section{Minimax Solution for Fixed Horizon}\label{sec:fh}
Although our primary interest in this paper is the case when the horizon is unknown to the learner, 
we first present some preliminary results on the setting where the horizon is known to both the learner and the adversary ahead of time. 
These will later be useful for the unknown horizon case.

If we treat the learner as an algorithm $\Alg$ that takes the information of previous rounds as inputs, 
and outputs a distribution $\P_t = \Alg(\P_{1:t-1}, \Z_{1:t-1})$ that the learner is going to play with, 
then finding the optimal solution in this fixed horizon setting can be viewed as solving the minimax expression
\begin{equation} \label{equ:minimax_fix}
\adjustlimits\inf_{\Alg}\sup_{\Z_{1:T}} \Reg(L_T, \M_T) .
\end{equation} 
Alternatively, we can recursively define: 
\begin{align*}
V(\M, 0) &\triangleq -\min_i M_i\ ; \\
V(\M, r) &\triangleq \min_{\P\in \Delta(N)} \max_{\Z\in\LS}  \(\P\cdot\Z + V(\M + \Z, r - 1)\) ,
\end{align*}
where $\M \in \R^N$ is a loss vector, $r$ is a nonnegative integer, 
and $\Delta(N)$ is the $N$ dimensional simplex. 
By a simple argument, one can show that the value of $V(\M, r)$ is the regret of a game with $r$ rounds starting from the situation that each action has initial loss $M_i$, 
and assuming both the learner and the adversary will play optimally. 
In fact, the value of Eq. \eqref{equ:minimax_fix}  is exactly $V(\0, T)$, 
and the optimal learner algorithm is the one that chooses the $\P^*$ which realizes the minimum in the definition of $V(\M, r)$ when the actions' cumulative loss vector is $\M$ and there are $r$ rounds left. 
We call $V(\0, T)$ the {\it value} of the game. 


As a concrete illustration of these ideas,
we now consider the basis vector loss space
\footnote{For other loss spaces, finding minimax solutions seems difficult.
However, we show the relation of the values of the game for different loss spaces
in the supplementary file, see Theorem \ref{thm:loss_space}. }, 
that is, $\LS = \{\e_1, \ldots, \e_N\}$. 
It turns out that under this loss space, the value function $V$ has a nice closed form. 
Similar to the results from \citet{CesabianchiFrHeHaScWa97} and \citet{AbernethyWaYe08}, 
we show that $V$ can be expressed in terms of a random walk. 
Suppose $R(\M, r)$ is the expectation of the loss of the best action if the adversary chooses each $\e_i$ uniformly randomly for the remaining $r$ rounds, 
starting from loss vector $\M$. 
Formally, $R(\M, r)$ can be defined in a recursive way:
$ R(\M, 0) \triangleq \min_i M_i \ ;
R(\M, r) \triangleq \frac{1}{N} \sum_{i=1}^N R(\M + \e_i, r-1). $
The connection between $V$ and $R$, 
and the optimal algorithm are then shown by the following theorem.
\begin{theorem}\label{thm:minimax_fix}
If $\LS = \{\e_1, \ldots, \e_N\} $, then for any vector $\M$ and integer $r \geq 0$,
$$V(\M, r) = \frac{r}{N} - R(\M, r).$$
Let $c_N = \frac{1}{N}\sqrt{2(N-1)\ln N}$. Then the value of the game satisfies
\begin{equation}\label{equ:value_upper}
V(\0, T) \leq c_N\sqrt{T}.
\end{equation} 
Moreover, on round $t$, the optimal learner algorithm is the one that chooses weight
$ P_{t, i} = V(\M_{t-1},  r) - V(\M_{t-1} + \e_i, r - 1)$
for each action $i$, where $\M_{t-1}$ is the current cumulative loss vector and $r$ is the number of remaining rounds, that is, $r = T - t + 1$.
\end{theorem}
Theorem \ref{thm:minimax_fix} tells us that under the basis vector loss space, 
the best way to play is to assume that the adversary is playing uniformly randomly, 
because $r/N$ and $R(\M, r)$ are exactly the expected losses for the learner and 
for the best action respectively.
In practice, computing $R(\M, r)$ needs exponential time. 
However, we can estimate it by sampling (see similar work in \citealp{AbernethyWaYe08}). 
Note that $c_N$ is decreasing when $N\geq4$ (with maximum value about $0.72$).
So contrary to the $O(\sqrt{T\ln N})$ regret bound for the general loss space $[0,1]^N$
which is increasing in $N$,
here $V(\0, T)$ is of order $O(\sqrt{T})$.

\section{Playing without Knowing the Horizon}\label{sec:uh}
We turn now to the case in which the horizon $T$ is unknown to the learner, 
which is often more realistic in practice. 
There are several ways of modeling this setting. 
For example, the horizon can be chosen ahead of time according to some fixed distribution, 
or it can even be chosen by the adversary. 
We will discuss these two variants separately. 

\subsection{Random Horizon}\label{subsec:rh}
Suppose the horizon $T$ is chosen according to some fixed distribution $Q$ 
which is known to both the learner and the adversary. 
Before the game starts, a random $T$ is drawn, 
and neither the learner nor the adversary knows the actual value of $T$. 
The game stops after $T$ rounds, 
and the learner aims to minimize the expectation of the regret. 
Using our earlier notation, the problem can be formally defined as
 $$\adjustlimits\inf_{\Alg}\sup_{\Z_{1:\infty}} \E_{T\sim Q}[\Reg(L_T, \M_T)],$$
where we assume the expectation is always finite.
We sometimes omit the subscript $T\sim Q$ for simplicity.

Continuing the example in Section~\ref{sec:fh} of the basis vector loss space,
we can again show the exact minimax solution,
which has a strong connection with the one for the fixed horizon setting. 

\begin{theorem}\label{thm:minimax_ran}
If $\LS = \{\e_1,\ldots,\e_N\}$, then 
\begin{equation}\label{equ:fix_and_ran}
\begin{split}
\adjustlimits \inf_{\Alg}\sup_{\Z_{1:\infty}} \E_{T\sim Q}[\Reg(L_T, \M_T)] \\ 
= \E_{T\sim Q}[\adjustlimits\inf_{\Alg}\sup_{\Z_{1:T}}\Reg(L_T, \M_T)]. 
\end{split}
\end{equation}
Moreover, on round $t$, the optimal learner plays with the distribution
$ \P_t = \E_{T\sim Q}[\P^T_t | T \geq t],$
where $\P^T_t$ is the optimal distribution the learner would play if the horizon is $T$, that is,
$ P^T_{t,i} = V(\M_{t-1}, T - t + 1) - V(\M_{t-1} + \e_i, T - t) .$
\end{theorem}

Eq. \eqref{equ:fix_and_ran} tells us that if the horizon is drawn from some distribution, 
then even though the learner does not know the actual horizon before playing the game, 
as long as the adversary does not know this information either, 
it can still do as well as the case when they are both aware of the horizon. 

However, so far this model does not seem to be quite useful in practice for several reasons. First of all, the horizon might not be chosen according to a distribution.
Even if it is, this distribution is probably unknown. 
Secondly, what we really care about is the performance which holds uniformly for any horizon,
instead of the expected regret. 
Last but not least, one might conjecture that the similar result stated in 
Theorem~\ref{thm:minimax_ran} should hold for other more general loss spaces,
which is in fact not true 
(see Example~\ref{exm:minimax_ran} in the supplementary file),
making the result seem even less useful. 

Fortunately, we address all these problems and develop new adaptive algorithms based on the result in this section.
We discuss these in Section \ref{sec:alg} after first introducing the fully adversarial model.

\subsection{Adversarial Horizon}
The most adversarial setting is the one where the horizon is completely controlled by the adversary. 
That is, we let the adversary decide whether to continue or stop the game on each round
according to the current situation.
However, notice that the value of the game is increasing in the horizon. 
So if the adversary can determine the horizon and its goal is still to maximize the regret, 
then the problem would not make sense 
because the adversary would clearly choose to play the game forever and never stop
leading to infinite regret. 
One reasonable way to address this issue is to scale the regret by the value of the fixed horizon game $V(\0, T)$,
so that the scaled regret $\Reg(L_T, \M_T) / V(\0, T)$ indicates how many times worse is the regret compared to the one that is optimal given the horizon. 
Under this setting, the corresponding minimax expression is
\begin{equation} \label{equ:minimax_var}
\V =  \adjustlimits\inf_{\Alg}\sup_T \sup_{\Z_{1:T}} \frac{\Reg(L_T, \M_T)}{V(\0,T)} .
\end{equation} 

Unfortunately, finding the minimax solution to this setting seems to be quite challenging, 
even for the simplest case $N = 2$. 
It is clear, however, that $\V$ is at most some constant 
due to the existence of adaptive algorithms such as the doubling trick,
which can achieve the optimal regret bound up to a constant without knowing $T$.
Another clear fact is $\V \geq 1$, 
since it is impossible for the learner to do better than the case when it is aware of the horizon. Below, we derive a nontrivial lower bound that is greater than $1$, 
thus proving that the adversary does gain strictly more power when it can stop the game whenever it wants. 
\begin{theorem}\label{thm:lower_bound}
If $N = 2$ and $\LS=[0,1]^2$, then $\V \geq \sqrt{2}$. 
That is, for every algorithm, there exists an adversary and a horizon $T$ such that the 
regret of the learner after $T$ rounds is at least $\sqrt{2}V(\0, T)$.
\end{theorem}

\section{A New General Adaptive Algorithm}\label{sec:alg}
We study next how the random-horizon algorithm of Section \ref{subsec:rh} can be used when the horizon is entirely unknown and furthermore, for a much more general
class of online learning problems. 
In Theorem \ref{thm:minimax_ran}, we
proposed an algorithm that simply takes the conditional expectation of the distributions we would have played if the horizon were given.  Notice that
even though it is derived from the random horizon setting, it can still be used in any setting as an adaptive algorithm in the sense that it does not
require the horizon as a parameter. However, to use this algorithm, we should ask two questions:
What distribution should we use? 
And what can we say about the algorithm's performance for an arbitrary horizon instead of in expectation? 

%

As a first attempt, suppose we use a uniform distribution over $1,\ldots,T_0$, 
where $T_0$ is a huge integer. 
From what we observe in some numerical calculations, $\E[\P^T_t | T \geq t]$ tends to be a uniform distribution in this case. 
Clearly it cannot be a good algorithm if for each round, 
it just places equal weights for each action regardless of the actions' behaviors. 
In fact, one can verify that the exponential distribution (that is, $\Pr[T=t] \propto \alpha^t$ for some constant $0 <\alpha<1$) also does not work. 
These examples show that even though this algorithm gives us the optimal expected regret, 
it can still suffer a big regret for a particular trial of the game, 
which we definitely want to avoid.  

Nevertheless, it turns out that there does exist a family of distributions that can guarantee the regret to be of order $O(\sqrt{T})$ for any $T$. 
Moreover, this is true for a very general online learning problem
that includes the Hedge setting we have been discussing.
Before stating our results, we first formally describe this
general setting, which is sometimes called the
{\it online convex optimization} problem \cite{Zinkevich03, Shalevshwartz11}. 
Let $S$ be a compact convex set, and $\F$ be a set of convex functions defined on $S$.
On each round $t=1,\ldots,T$: 
(1) the learner chooses a point $\x_t \in S$; 
(2) the adversary chooses a loss function $f_t \in \F$; 
(3) the learner suffers loss $f_t(\x_t)$ for this round. 
The regret after $T$ rounds is defined by
$$ \Reg(\x_{1:T}, f_{1:T}) = \sum_{t=1}^T f_t(\x_t) 
- \min_{\x \in S} \sum_{t=1}^T f_t(\x). $$
It is clear that the Hedge problem is a special case of the above setting with
$S$ being the probability simplex, and $\F$ being a set of linear functions
defined by a point in the loss space, that is, $\F = \{f(\x) = \x\cdot \w : \w \in \LS\}$.
Similarly, to study the minimax algorithm we define the following $V_{S,\F}$ 
function of the multiset $\MM$ of loss functions we have encountered 
and the number of remaining rounds $r$:
\begin{align*}
V_{S,\F}(\MM, 0) &\triangleq -\min_{\x\in S} \sum_{f \in \MM} f(\x) ; \\
V_{S,\F}(\MM, r) &\triangleq \min_{\x \in S} \max_{f\in\F}  
\(f(\x) + V_{S,\F}(\MM \uplus \{f\}, r - 1)\),
\end{align*}
where $\uplus$ denotes multiset union.
We omit the subscript of $V_{S,\F}$ whenever there is no confusion.
Let $\x_t^T$ be the output of the minimax algorithm on round $t$.
In other words, $\x_t^T$ realizes the minimum in the definition of
$V(f_{1:t-1}, T-t+1)$. 
We will adapt the idea in Section~\ref{subsec:rh} and study the adaptive 
algorithm that outputs $\E_{T\sim Q}[\x_t^T | T \geq t] \in S$ 
on round $t$ for a distribution $Q$ on the horizon. 
One mild assumption needed is 
\begin{assumption}\label{assu:decrease}
$\forall \MM$ and $r>0$, $V(\MM, r) \geq V(\MM, 0)$ .
\end{assumption}
Roughly speaking, this assumption implies that the game is in
the adversary's favor: playing more rounds leads to greater regret.
It holds for the Hedge setting with basis vector loss space 
(see Property~\ref{pro:monotonicity_r} in the supplementary file).
In fact, it also holds as long as $\F$ contains the zero function 
$f_0(\x) \equiv 0$. To see this, simply observe that
\begin{align*}
V(\MM, r) &=  \adjustlimits\min_{\x \in S} \max_{f\in\F}  
\(f(\x) + V(\MM \uplus \{f\}, r - 1)\) \\
&\geq  V(\MM \uplus \{f_0\}, r - 1) \\
&\geq \ldots \geq V(\MM \uplus \{f_0,\ldots,f_0\}, 0)
= V(\MM, 0).
\end{align*}
So the assumption is mild and will hold for all the examples we consider.

Below, we first give a general upper bound on the regret that holds for 
any distribution and has no dependence on the choices of the adversary. 
After that we will show what the appropriate distributions are to make this 
 bound $O(\sqrt{T})$.

\begin{theorem}\label{thm:adaptive_gen_upper}
Let $\BV_t(\MM) = \E_{T\sim Q}[V(\MM, T - t + 1) | T \geq t]$
and $q_t = \Pr_{T\sim Q}[T = t | T \geq t]$. 
Suppose Assumption~\ref{assu:decrease} holds, and 
on round $t$ the learner chooses $\x_t = 
\E_{T\sim Q}[\x_t^T | T \geq t] $ where $\x_t^T$ is the output
of the minimax algorithm as described above.
Then for any $T_s$, the regret after $T_s$ rounds is at most
$$ \BV_1(\emptyset) + \sum_{t=1}^{T_s} q_t \BV_{t+1}(\emptyset).$$
\end{theorem}
To prove Theorem~\ref{thm:adaptive_gen_upper}, we
first show the following lemma.
\begin{lemma}\label{lem:separate}
For any $r\geq 0$ and multiset $\MM_1$ and $\MM_2$,
\begin{equation}\label{equ:separate}
V(\MM_1 \uplus \MM_2, r) - V(\MM_1, 0) \leq V(\MM_2, r). 
\end{equation}
\end{lemma}
\begin{proof}
If $r = 0$, then Eq.~\eqref{equ:separate} holds since
$$ \min_{\x\in S} \sum_{f \in \MM_1} f(\x) +
\min_{\x\in S} \sum_{f \in \MM_2} f(\x)
\leq \min_{\x\in S} \sum_{f \in \MM_1 \uplus \MM2} f(\x).$$
Now assume Eq.~\eqref{equ:separate} holds for $r-1$. 
By induction one has
\begin{align*}
&V(\MM_1 \uplus \MM_2, r) - V(\MM_1, 0) \\
=& \min_{\x \in S} \max_{f\in\F}  
\(f(\x) + V(\MM_1 \uplus \MM_2 \uplus \{f\}, r - 1)\) - V(\MM_1, 0) \\
\leq& \min_{\x \in S} \max_{f\in\F}  
\(f(\x) + V(\MM_2 \uplus \{f\}, r - 1)\) = V(\MM_2, r),
\end{align*}
concluding the proof.
\end{proof}

\begin{proof}[Proof of Theorem~\ref{thm:adaptive_gen_upper}]
By definition of $\x_t^T$, we have
\begin{align*}
&V(f_{1:t-1}, T-t+1) \\
=& \max_{f\in \F}(f(\x_t^T) + V(f_{1:t-1}\uplus \{f\}), T-t) \\
\geq& f_t(\x_t^T) + V(f_{1:t}, T-t).
\end{align*}
Therefore, by convexity and the fact that 
$\Pr[T=t' | T\geq t] = (1-q_t)\Pr[T=t' | T\geq t+1]$ for any $t'>t$,
the loss of the algorithm on round $t$ is
\begin{align*}
&f_t(\x_t) = f_t(\E[\x_t^T | T \geq t])
\leq \E[f_t(\x_t^T) | T \geq t] \\
\leq& \E[V(f_{1:t-1}, T-t+1) - V(f_{1:t}, T-t) | T \geq t] \\
=& \BV_t(f_{1:t-1}) - q_t V(f_{1:t}, 0) - (1-q_t)\BV_{t+1}(f_{1:t}) \\
\leq& \BV_t(f_{1:t-1}) - \BV_{t+1}(f_{1:t}) + q_t \BV_{t+1}(\emptyset),
\end{align*}
where the last equality holds because 
$\BV_{t+1}(f_{1:t}) - V(f_{1:t}, 0) = \E[V(f_{1:t}, T-t) - V(f_{1:t},0) | T \geq t + 1]
\leq \E[V(\emptyset, T-t) | T \geq t + 1] = \BV_{t+1}(\emptyset)$ 
by Lemma~\ref{lem:separate}.
We conclude the proof by summing up $f_t(\x_t)$ over $t=1,\ldots,T_s$ 
and pointing out that $\BV_{T_s+1}(f_{1:T_s}) = \E[V(f_{1:T_s}, T-T_s)| T \geq T_s + 1]
\geq \E[V(f_{1:T_s}, 0)| T \geq T_s + 1] = -\min_{\x\in S}\sum_{t=1}^{T_s} f_t(\x_t)$
by Assumption~\ref{assu:decrease}.
\end{proof}

As a direct corollary, we now show an appropriate choice of $Q$. 
We assume that the optimal regret under the fixed horizon setting
is of order $O(\sqrt{T})$. That is:
\begin{assumption}\label{assu:sqrt}
For any $T$, $V(\emptyset, T) \leq c_N\sqrt{T}$ for some constant $c_N$ that
might depend on $N$.
\end{assumption}
This is proven to be true in the literature for all the examples we consider, 
especially when $\F$ contains linear functions.
\begin{theorem}\label{thm:adaptive_upper}
Under Assumption~\ref{assu:sqrt} and the same conditions of 
Theorem \ref{thm:adaptive_gen_upper}, 
if $\Pr[T = t] \propto 1/t^d$ where $d > \frac{3}{2}$ is a constant, 
then for any $T_s$, the regret after $T_s$ rounds is at most
$$ \frac{\Gamma(d-\frac{3}{2})}{\Gamma(d)}(d-1)^2c_N\sqrt{\pi T_s} + o(\sqrt{T_s}),$$
where $\Gamma$ is the gamma function.
Choosing $d \approx 2.35$ approximately minimizes the main term in the bound, 
leading to regret approximately $3c_N\sqrt{T_s} + o(\sqrt{T_s})$.
\end{theorem}

Theorem \ref{thm:adaptive_upper} tells us that pretending that the horizon is drawn from the distribution $\Pr[T = t] \propto 1/t^d\ (d>3/2)$ can always achieve low regret,
even if the actual horizon $T_s$ is chosen adversarially.  
Also notice that the constant $3$ in the bound for the term $c_N\sqrt{T_s}$ is less than the one for the doubling trick with the fixed horizon optimal algorithm, 
which is $2 + \sqrt{2}$ \cite{CesabianchiLu06}.  
We will see in Section \ref{subsec:ball_game} an experiment 
showing that our algorithm performs much better than the doubling trick.

It is straightforward to apply our new algorithm to different instances
of the online convex optimization framework.
Examples include Hedge with basis vector loss space, predicting with 
expert advice \cite{CesabianchiFrHeHaScWa97},
online linear optimization within an $\ell_2$ ball \cite{AbernethyBaRaTe08}
or an $\ell_\infty$ ball \cite{McMahanAb13}. 
These are examples where minimax algorithms for fixed horizon are
already known. In theory, however, our algorithm is still applicable 
when the minimax algorithm is unknown, such as Hedge with the general loss
space $[0,1]^N$.

\section{Implementation and Applications}\label{sec:gen}
In this section, we discuss the implementation issue of our new algorithm,
and also show that the idea of using a ``pretend prior distribution'' 
is much more applicable in online learning than we have discussed so far. 

\subsection{Closed Form of the Algorithm}\label{subsec:ball_game}
Among the examples listed at the end of Section~\ref{sec:alg}, 
we are especially interested in
online linear optimization within an $\ell_2$ ball since our 
algorithm enjoys an explicit closed form in this case.
Specifically, we consider the following problem 
(all the norms are $\ell_2$ norms):
take $S= \{\x \in \R^N : \|\x\| \leq 1\}$,
and $\F = \{f(\x) = \x \cdot \w : \w \in S\}$.
In other words, the adversary also chooses a point in $S$ on each round,
which we denote by $\w_t$. 
\citet{AbernethyBaRaTe08} showed a simple but exact minimax optimal algorithm 
for the fixed horizon setting (for $N>2$): on each round $t$, choose
\begin{equation}\label{equ:ball_game_minimax}
\x_t^T = \left. -\W_{t-1} \middle/ \sqrt{\|\W_{t-1}\|^2 + (T-t+1)} \right.,
\end{equation}
where $\W_t = \sum_{t'=1}^t \w_{t'}$.
This strategy guarantees the regret to be at most $\sqrt{T}$.
To make this algorithm adaptive, we again assign a distribution over the horizon.
However, in order to get an explicit form for $\E[\x_t^T | T\geq t]$,
a continuous distribution on $T$ is necessary. 
It does not seem to make sense at first glance since the horizon is always an integer,
but keep in mind that the random variable $T$ is merely an artifact of our algorithm,
and Eq.~\eqref{equ:ball_game_minimax} is well defined with $T\geq t$ 
being a real number.
As long as the output of the learner is in the set $S$, our algorithm is valid.
The analysis for our algorithm also holds with minor changes.
Specifically, we show the following:
\begin{theorem}\label{thm:ball_game}
Let $T \geq 1$ be a continuous random variable with probability density $f(T) \propto 1/T^2$. 
If the learner chooses $\x_t = \E[\x_t^T | T\geq t]$ on round t, 
where $\x_t^T$ is defined by Eq. \eqref{equ:ball_game_minimax},
then the regret after $T_s$ rounds is at most $\pi \sqrt{T_s} + o(\sqrt{T_s})$ for any $T_s$.
Moreover, $\x_t$ has the following explicit form
\begin{equation}\label{equ:ball_game_adaptive}
\x_t 
= \begin{cases}
\(\frac{t\cdot \tanh^{-1}\(\sqrt{1-t/c}\)}{(c-t)^{3/2}}-\frac{\sqrt{c}}{c-t}\) \W_{t-1} 
&\text{\quad if $c\neq t$} \\
- \frac{2t}{3c^{3/2}}\W_{t-1} &\text{\quad else,}
\end{cases}
\end{equation} 
where $c = 1 + \|\W_{t-1}\|^2$.
\end{theorem}

The algorithm we are proposing in Eq. \eqref{equ:ball_game_adaptive} looks quite inexplicable 
if one does not realize that it comes from the expression $\E[\x_t^T | T\geq t]$ with an appropriate distribution. 
Yet the algorithm not only enjoys a low theoretic regret bound as shown in Theorem \ref{thm:ball_game},
but also achieves very good performance in simulated experiments.

To show this, we conduct an experiment that compares the regrets of four algorithms
at any time step within 1000 rounds 
against an adversary that chooses points in $S$ uniformly at random ($N=10$).
The results are shown in Figure \ref{fig:ballgame}, 
where each data point is the maximum regret over 1000 randomly generated adversaries
for the corresponding algorithm and horizon.
The four algorithms are: 
the minimax algorithm in Eq. \eqref{equ:ball_game_minimax} (OPT);
the one we proposed in Theorem \ref{thm:ball_game} (DIST);
online gradient descent, a general algorithm for online optimization (see \citealp{Zinkevich03}) (OGD);
and the doubling trick with the minimax algorithm (DOUBLE).
Note that OPT is not really an adaptive algorithm: 
it ``cheats'' by knowing the horizon $T=1000$ in advance, 
and thus performs best at the end of the game.
We include this algorithm merely as a baseline.
Figure \ref{fig:ballgame} shows that our algorithm DIST achieves consistently much lower regret than any other adaptive algorithm,
including OGD which seems to enjoy a better constant in the regret bound
($2\sqrt{2T_s}$, see \citealp{Zinkevich03}).
Moreover, for the first 450 rounds or so, our algorithm performs even better
than OPT, implying that using the optimal algorithm with a large guess on the
horizon is inferior to our algorithm.
Finally, we remark that although the doubling trick is widely applicable in theory,
in experiments it is beaten by most of the other algorithms.

\begin{figure}[ht]
\vskip 0.2in
\begin{center}
\centerline{\includegraphics[scale=0.45]{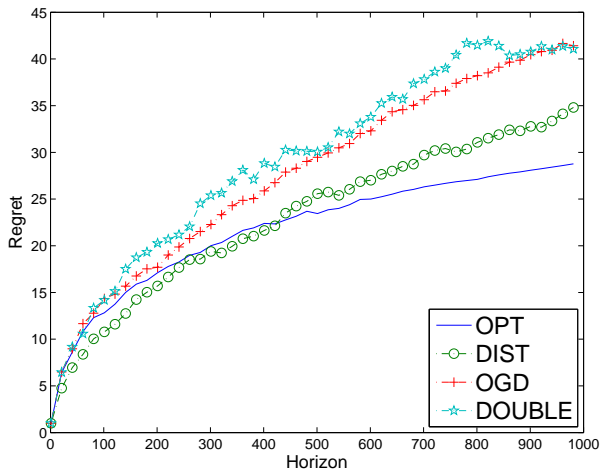}}
\caption{Comparison of four algorithms.}
\label{fig:ballgame}
\end{center}
\vskip -0.2in
\end{figure} 

\subsection{Randomized Play and Efficient Implementation}\label{subsec:random_play}
Implementation is an issue for our algorithm when there is no closed form 
for $\E[\x_t^T | T \geq t]$, which is usually the case.
One way to address this problem is to compute the sum of the first 
sufficient number of  terms in the series, which can be a good estimate 
since the weight for each term decreases rapidly. 

However, there is another more natural way to deal with the implementation
issue when we are in a similar setting but allowed to play randomly.
Specifically, consider a modified Hedge setting where on each round $t$,
the learner can bet on one and only one action $I_t$,  
and then the loss vector $\Z_t \in [0,1]^N$ is revealed with the learner suffering
loss $Z_{t, I_t}$ for this round.
It is well known that in this kind of problem, randomization is necessary
for the learner to achieve sub-linear regret.  That is, $I_t$ is a random
variable and $\Z_t$ is decided without knowing the actual draw of $I_t$.
In addition, suppose $\P_t$, the conditional distribution of $I_t$ given 
the past, only depends on $\Z_{1:t-1}$, and the learner 
achieves sub-linear regret in the usual Hedge setting 
(sometimes called {\it pseudo-regret}):
\begin{equation}\label{equ:oblivious}
\sum_{t=1}^{T} \P_t \cdot \Z_t  - \min_i M_{T,i} \leq c_N\sqrt{T} 
\end{equation}
(recall $\M_t = \sum_{t'=1}^t \Z_{t'}$)
for any $\Z_{1:T}$ and a constant $c_N$.
Then the learner also achieves sub-linear regret with high probability 
in the randomized setting. 
That is, with probability at least $1-\delta$, the actual regret satisfies:
\begin{equation}\label{equ:nonoblvious}
\sum_{t=1}^{T} Z_{t, I_t}  - \min_i M_{T,i} \leq c_N\sqrt{T} 
+ \sqrt{\frac{T}{2}\ln \frac{1}{\delta}}.  
\end{equation}
We refer the interested reader to Lemma 4.1 of \citet{CesabianchiLu06}
for more details.

Therefore, in this setting we can implement our algorithm in an efficient
way: on round $t$, first draw a horizon $T\geq t$ according to distribution
$\Pr[T=t'] \propto 1/t'^d$, then draw $I_t$ according to $\P_t^T$.
It is clear that the marginal distribution of $I_t$ of this process is exactly
$\E[\P_t^T | T \geq t]$. 
Hence, Eq.~\eqref{equ:oblivious} is satisfied by Theorem~\ref{thm:adaptive_upper}
and as a result Eq.~\eqref{equ:nonoblvious} holds. 

\subsection{Combining with the FPL algorithm}\label{subsec:FPL}
Even if we have an efficient randomized implementation,
or sometimes even have a closed form of the output, 
it is still too constrained if we can only apply our technique 
to minimax algorithms since they are usually difficult to derive
and sometimes even inefficient to implement. 
It turns out, however,  that the ``pretend prior distribution'' idea
is applicable for many other non-minimax algorithms, 
which we will discuss from this section on.

Continuing the randomized setting discussed in the previous section,
we study the well-known ``follow the perturbed leader (FPL)'' algorithm
\cite{KalaiVe05},
which chooses $I_t \in \arg\min_i (M_{t-1,i} + \xi_{t, i})$
where $\bxi_t \in R^N$ is a random variable drawn from some distribution.
This distribution sometimes requires the horizon $T$ as a parameter.
If this is the case, applying our technique would have a simple 
{\it Bayesian interpretation}: put a prior distribution on an unknown parameter
of another distribution.
Working out the marginal distribution of $\bxi_t$ would then give
an adaptive variant of FPL.

For simplicity, consider drawing $\bxi_t^T$ uniformly at random from the
hypercube $[0, \Delta_T]^N$ (see Chapter 4.3 of \citealp{CesabianchiLu06}).
If $\Delta_T = \sqrt{TN}$, then the pseudo-regret is upper bounded
by $2\sqrt{TN}$ (whose dependence on $N$ is not optimal).
Now again let $T \geq 1$ be a continuous random variable with
probability density $f(T) \propto 1/T^d (d > 3/2)$, and $\bxi_t$ be obtained
by first drawing $T$ given $T \geq t$, and then drawing a point
uniformly from $[0, \Delta_T]^N$. We show the following:
\begin{lemma}\label{lem:FPL}
If $\Delta_t = \sqrt{btN}$ for some constant $b>0$,
the marginal density function of $\bxi_t$ is
\begin{equation}\label{equ:FPLdensity}
f_t(\bxi) \propto \begin{cases}
0 &\text{\quad if $\min_i \xi_i < 0$} \\
\min\left\{1,  \(\frac{\Delta_t}{\|\bxi\|_\infty}\)^{2d-2+N}\right\} &\text{\quad else.}
\end{cases}
\end{equation}
The normalization factor is $\frac{d-1}{d-1+N/2} \Delta_t^{-N}$.
\end{lemma}
\begin{theorem}\label{thm:FPL}
Suppose on round $t$, the learner chooses
$$ I_t \in \arg\min_i (M_{t-1,i} + \xi_{t, i}) ,$$
where $\bxi_t$ is a random variable with density function~\eqref{equ:FPLdensity}.
Then the pseudo-regret after $T_s$ rounds is at most
$$ \(\frac{d-1}{\sqrt{b}(d-1/2)} + \frac{\sqrt{b}(d-1)^2}{d-3/2}\)2\sqrt{T_sN} .$$
Choosing $b = \frac{d-3/2}{(d-1/2)(d-1)}$ and $d = 1 + \frac{\sqrt{3}}{2}$ 
minimizes the main term in the bound, leading to about $4.6\sqrt{T_sN}$.
\end{theorem}
By the exact same argument, the actual regret is bounded by
the same quantity plus $\sqrt{\frac{T}{2}\ln \frac{1}{\delta}}$ with
probability $1 -\delta$.

\subsection{Generalizing the Exponential Weights Algorithm}\label{subsec:exp_alg}
Now we come back to the usual Hedge setting and consider
another popular non-minimax algorithm
(note that it is trivial to generalize the results to the randomized setting).
When dealing with the most general loss space $[0,1]^N$,
the minimax algorithm is unknown even for the fixed horizon setting. 
However, generalizing the weighted majority algorithm of \citet{LittlestoneWa94}, 
\citet{FreundSc97, FreundSc99} presented an algorithm 
using exponential weights that can deal with this general loss space 
and achieve the $O(\sqrt{T\ln N})$ bound on the regret.
The algorithm takes the horizon $T$ as a parameter, and on round $t$,
it simply chooses $P_{t,i} \propto \exp(-\eta M_{t-1,i})$, 
where $\eta = \sqrt{(8\ln N)/T}$ is the learning rate. 
It is shown that the regret of this algorithm is at most $\sqrt{(T\ln N)/2}$. 
\citet{AuerCeGe02} proposed a way to make this algorithm adaptive by simply setting a time-varying learning rate $\eta = \sqrt{(8\ln N)/t}$, where $t$ is the current round, 
leading to a regret bound of $\sqrt{T\ln N}$ for any $T$ 
(see Chapter 2.5 of \citealp{Bubeck11}). 
In other words, the algorithm always treats the current round as the last round.
Below, we show that our ``pretend distribution'' idea can also be used to make this exponential weights algorithm adaptive, 
and is in fact a generalization of the adaptive learning rate algorithm by \citet{AuerCeGe02}.

\begin{theorem}\label{thm:adaptive_MW}
Let $\LS =[0,1]^N$,  $\Pr[T = t] \propto 1/t^d\ (d > 3/2)$ and $\eta_T = \sqrt{(b\ln N)/T}$, where $b$ is a constant. 
If on round $t$, 
the learner assigns weight $\E_{T\sim Q}[P^T_{t,i} | T \geq t]$ to each action $i$, 
where $P^T_{t,i} \propto \exp(-\eta_T M_{t-1,i})$, then for any $T_s$, 
the regret after $T_s$ rounds is at most
$$  \(\frac{\sqrt{b}(d-1)}{4(d-1/2)}  +\frac{d-1}{(d-3/2)\sqrt{b}}\)\sqrt{T_s\ln N} 
+ o(\sqrt{T_s\ln N}) .$$
Setting $b= \frac{4d-2}{d-3/2}$ minimizes the main term, 
which approaches $1$ as $d\ra\infty$.
\end{theorem}

Note that if $d\ra \infty$, our algorithm simply becomes the one of \citet{AuerCeGe02}, 
because $\Pr[T=\tau | T\geq t]$ is $1$ if $\tau=t$ and $0$ otherwise.
Therefore, our algorithm can be viewed as a generalization of the idea of treating the current round as the last round. 
However, we emphasize that the way we deal with unknown horizon is more applicable
in the sense that if we try to make a minimax algorithm 
adaptive by treating each round as the last round,
one can construct an adversary that leads to linear---
and therefore grossly suboptimal---regret,
whereas our approach yields nearly optimal regret.
(See Example~\ref{exm:last_round1} and~\ref{exm:last_round2} 
in the supplementary file for details.)

\subsection{First Order Regret Bound}
So far all the regret bounds we have discussed are in terms of the horizon, 
which are also called {\it zeroth order bounds}. 
More refined bounds have been studied in the literature \cite{CesabianchiLu06}. 
For example, the {\it first order bound} for Hedge, that depends on the loss of the best action $m^*$ at the end of the game, usually is of order $O(\sqrt{m^*\ln N})$.
Again, using the exponential weights algorithm with a slightly different learning rate 
$\eta = \ln(1+\sqrt{(2\ln N)/m^*})$,
one can show that the regret is at most $\sqrt{2m^*\ln N}+\ln N$.
Here, $m^*$ is prior information on the loss sequence similar to the horizon.
To avoid exploiting this information that is unavailable in practice, 
one can again use techniques like the doubling trick or the time-varying learning rate.
Alternatively, we show that the ``pretend distribution" technique can also be used here.
Again it makes more sense to assign a continuous distribution on the 
loss of the best action instead of a discrete one.

\begin{theorem}\label{thm:first_order}
Let $\LS =[0,1]^N$,  $m_t = \min_i M_{t,i} + 1$, $\eta_m = \sqrt{(\ln N)/m}$, 
and $m \geq 1$ be a continuous random variable with probability density $f(m) \propto 1/m^d \;(d >3/2)$. 
If on round $t$, the learner assigns weight $\E[P^m_{t,i} | m \geq m_{t-1}]$ 
to each action $i$,
where $P^m_{t,i} \propto \exp(-\eta_m M_{t-1,i})$, 
then for any $T_s$, the regret after $T_s$ rounds is at most
\begin{align*}
&\frac{3(d-7/6)(d-1)}{(d-3/2)(d-1/2)}\sqrt{m^* \ln N}  \\
+ &(1+(d-1)\ln(m^*+1))\ln N + o(\sqrt{m^*\ln N}),
\end{align*}
where $m^* = \min_i M_{T_s,i} $ is the loss of the best action after $T_s$ rounds.
Setting $d=5/2+\sqrt{2}$ minimizes the main term, which becomes
$(3/2+\sqrt{2})\sqrt{m^*\ln N} $.
\end{theorem}

\newpage
\bibliography{../../references/ref.bib}
\bibliographystyle{icml2014}

\newpage
\appendix
Through all the proofs, we denote the set $\{1, \ldots, m\}$ by $[m]$.

\section{Proof of Theorem \ref{thm:minimax_fix}}
\label{app:minimax_fix}

We first state a few properties of the function $R$:
\begin{proposition} \label{pro:property}
For any vector $\M$ of $N$ dimensions and integer $r$,

\begin{property}\label{pro:additive_R}
$R(\M, r) = a + R((M_1 - a,\ldots, M_N -a), r)$ for any real number $a$ and $r\geq 0$.
\end{property}

\begin{property}\label{pro:monotonicity_R}
$R(\M, r)$ is non-decreasing in $M_i$ for each $i = 1, \ldots, N$.
\end{property}

\begin{property}\label{pro:difference_r}
If $r > 0$,  $R(\M, r) - R(\M, r - 1) \leq 1/N$.
\end{property}

\begin{property}\label{pro:dist}
If $r > 0$, and $P_i = \frac{1}{N} + R(\M + \e_i, r - 1) - R(\M, r)$ for each $i = 1,
\ldots,N$,
then $\P = (P_1, \ldots, P_N)$ is a distribution in the simplex $\Delta(N)$.
\end{property}

\end{proposition}

\begin{proof}[Proof of Proposition \ref{pro:property}]
We omit the proof for Property \ref{pro:additive_R} and \ref{pro:monotonicity_R}, 
since it is straightforward. 
We prove Property \ref{pro:difference_r} by induction. For the base case $r = 1$, let $S = \{j: M_j = \min_i M_i\}$. If $|S| = 1$, then $R(\M + \e_i, 0)$ is $R(\M, 0)$
for $i \notin S$ and $R(\M, 0) + 1$ otherwise. If $|S| > 1$, then $R(\M + \e_i, 0)$ is simply $R(\M, 0)$ for all $i$. In either case, we have
\begin{align*}
R(\M, 1) &= \frac{1}{N} \sum_{i=1}^N R(\M + \e_i, 0) \leq  
\frac{1}{N} (1 + \sum_{i=1}^N R(\M, 0))  \\
&=  \frac{1}{N}  + R(\M, 0),
\end{align*}
proving the base case. Now for $r > 1$, by definition of $R$ and induction,
\begin{align*}
&R(\M, r) - R(\M, r - 1)  \\
=& \frac{1}{N}\sum_{j=1}^N \(R(\M+\e_i, r - 1) - R(\M + \e_i, r - 2)\)   \\
\leq&  \frac{1}{N}\sum_{j=1}^N \frac{1}{N} = \frac{1}{N}, 
\end{align*}
completing the induction. For Property \ref{pro:dist}, 
it suffices to prove $P_i \geq 0$ for each $i$ and $\sum_{i=1}^N P_i = 1$. 
The first part can be shown using Property \ref{pro:monotonicity_R} and \ref{pro:difference_r}:
\begin{align*}
P_i &= \frac{1}{N} + R(\M + \e_i, r-1) - R(\M, r)  \\
&\geq \frac{1}{N} + R(\M, r-1) - (\frac{1}{N}  + R(\M, r-1)) = 0 .
\end{align*}
The second part is also easy to show by definition of $R$:
\begin{align*}
\sum_{i = 1}^N P_i &= 1 + \sum_{i=1}^N R(\M + \e_i, r - 1)  - N R(\M, r)  \\
&= 1 + N R(\M, r) - N R(\M, r) = 1.
\end{align*}
\end{proof}

\begin{proof}[Proof of Theorem \ref{thm:minimax_fix}]
First inductively prove $V(\M, r) = r/N - R(\M, r)$ for any $r \geq 0$. 
The base case $r = 0$ is trivial by definition. 
For $r > 0$, 
\begin{align*}
V(\M, r) &= \min_{\P\in \Delta(N)} \max_{\Z\in\LS}  \(\P\cdot\Z + V(\M + \Z, r - 1)\) \\
&= \min_{\P\in \Delta(N)} \max_{i \in [N]}  \(P_i + V(\M + \e_i, r - 1)\)  \tag{LS = $\{\e_1, \ldots, \e_N\} $}  \\
&= \min_{\P\in \Delta(N)} \max_{i \in [N]}  \(P_i + \frac{r-1}{N} - R(\M + \e_i, r - 1)\)  \tag{by induction} 
\end{align*}
Denote $P_i + (r-1)/N - R(\M + \e_i, r - 1)$ by $g(\P,i)$. 
Notice that the average of $g(\P,i)$ over all $i$ is irrelevant to $\P$ :
$\frac{1}{N}\sum_{i=1}^N g(\P,i) = r/N - R(\M, r) $. 
Therefore,  $\max_i g(\P,i) \geq r/N - R(\M, r)$ for any $\P$, and 
\begin{equation}\label{equ:lower_fix}
V(\M, r) = \min_\P \max_i g(\P,i) \geq r/N - R(\M, r).
\end{equation} 
On the other hand, from Proposition \ref{pro:property}, we know that 
$P^*_i = 1/N + R(\M + \e_i, r - 1) - R(\M, r) \; (i\in [N])$ is a valid distribution. Also,
\begin{equation}\label{equ:upper_fix}
\begin{split}
V(\M, r) = \min_\P \max_i g(\P,i) \leq \max_i g(\P^*,i) \\
= \max_i \(\frac{r}{N} - R(\M, r)\) = \frac{r}{N} - R(\M, r). 
\end{split}
\end{equation} 
So from Eq. \eqref{equ:lower_fix} and \eqref{equ:upper_fix} we have $V(\M, r) = r/N - R(\M, r)$, and also
$P^*_i = 1/N + R(\M + \e_i, r - 1) - R(\M, r) = V(\M,r) - V(\M+\e_i, r-1)$ 
realizes the minimum, and thus is the optimal strategy. \\

It remains to prove $V(\0, T) \leq c_N\sqrt{T}$. 
Let $\Z_1, \ldots, \Z_T$ be independent uniform random variables taking values 
in $\{\e_1, \ldots, \e_N \}$.
By what we proved above, 
$$ V(\0,T) = \frac{T}{N} - \E[\min_{i\in [N]} \sum_{t=1}^T Z_{t,i}]
= \E[\max_{i\in [N]} \sum_{t=1}^T(1/N-Z_{t,i})].
$$
Let $y_{t,i} = 1/N-Z_{t,i}$. 
Then each $y_{t, i}$ is a random variable that takes
value $1/N$ with probability $1-1/N$ and $1/N-1$ with probability $1/N$.
Also, for a fixed $i$, $y_{1,i}, \ldots, y_{T,i}$ are independent
(note that this is not true for $y_{t,1}, \ldots, y_{t,N}$ for a fixed $t$).
It is shown in Lemma 3.3 of \citet{BerendKo13}
that each $y_{t,i}$ satisfies
$$ \E[\exp(\lambda y_{t,i})] \leq \exp(\frac{\lambda^2 \sigma^2}{2}), \;\forall \lambda > 0,$$
where $\sigma^2 = (N-1)/N^2$ is the variance of $y_{t,i}$.
So if we let $Y_i = \sum_{t=1}^T y_{t,i}$, by the independence of each term, we have
$\forall \lambda > 0,$
\begin{align*}
\E[\exp(\lambda Y_{i})] 
&=  \E[\prod_{t=1}^T\exp(\lambda y_{t,i})]
= \prod_{t=1}^T \E[\exp(\lambda y_{t,i})] \\
&\leq \exp(\frac{\lambda^2 \sigma^2 T}{2}).
\end{align*}
Now using Lemma A.13 from \citet{CesabianchiLu06}, we arrive at
$$ \E[\max_{i\in[N]} Y_i] \leq \sigma \sqrt{2T\ln N} = c_N\sqrt{T}.$$
We conclude the proof by pointing out
$$V(\0,T)  = \E[\max_{i\in [N]} \sum_{t=1}^T(1/N-Z_{t,i})] = \E[\max_{i\in[N]} Y_i] 
\leq c_N\sqrt{T}. $$
\end{proof}

As a direct corollary of Proposition~\ref{pro:property} 
and Theorem~\ref{thm:minimax_fix}, below
we list a few properties of the function $V$ for later use.
\begin{proposition}\label{pro:minimax_fix}
If $\LS = \{\e_1, \ldots, \e_N\}$, then for any vector $\M$ and integer $r$, 

\begin{property}\label{pro:additive}
$V(\M, r) = V((M_1 - a,\ldots, M_N -a), r) - a$ for any real number $a$ and $r\geq 0$.
\end{property}

\begin{property}\label{pro:monotonicity_M}
$V(\M, r)$ is non-increasing in $M_i$ for each $i = 1, \ldots, N$.
\end{property}

\begin{property}\label{pro:monotonicity_r}
$V(\M, r)$ is non-decreasing in $r$. 
\end{property}

\end{proposition} 


\section{Proof of Theorem \ref{thm:minimax_ran}}\label{app:minimax_ran}
\begin{proof}
Define $\BV_t(\M) = \E[V(\M, T - t + 1) | T \geq t]$ 
and $q(t', t) = \Pr[T=t'| T\geq t]$.
We will prove an important property of the $\BV$ function:
\begin{equation}\label{equ:minimax_ran_recur}
\begin{split}
\BV_t(\M) = \min_{\P\in \Delta(N)} \max_{i\in [N]}  ( P_i + q(t, t)V(\M + \e_i, 0) +\\
 \(1-q(t,t)\) \BV_{t+1}(\M + \e_i)).
\end{split}
\end{equation}
This equation shows that $\BV_t(M)$ is the conditional expectation of the regret 
given $T \geq t$, 
starting from cumulative loss vector $\M$ 
and assuming both the learner and the adversary are optimal.
This is similar to the function $V$ in the fixed horizon case, 
and again the value of the game 
$\inf_{\Alg}\sup_{\Z_{1:\infty}} \E_{T\sim Q}[\Reg(L_T, \M_T)]$
is simply $\BV_1(\0)$. \\

To prove Eq. \eqref{equ:minimax_ran_recur}, 
we plug the definition of $\BV_{t+1}(\M+\e_i)$
into the right hand side and get $\min_{\P} \max_{i}  g(\P, i)$
where $g(\P, i) =  
 P_i + q(t, t)V(\M + \e_i, 0) + \(1-q(t,t)\) \E[V(\M+\e_i, T - t) | T \geq t + 1] $.
Using the fact that for any $t' \geq t + 1$,
\begin{align*}
&(1 - q(t, t)) q(t', t + 1) \\
=& \Pr[T>t|T\geq t] \Pr[T = t' | T \geq t + 1] \\
=& \Pr[T=t' | T\geq t] = q(t', t),
\end{align*}
$g(\P, i)$ can be simplified in the following way:
\begin{align}
&g(\P, i)  \\
=& P_i + q(t, t)V(\M + \e_i, 0) +   \notag \\
&\(1-q(t,t)\) \sum_{T=t+1}^\infty \(q(T, t+1)V(\M+\e_i, T - t)\)  \notag\\
=& P_i + q(t, t)V(\M + \e_i, 0) +  \notag \\
&\sum_{T=t+1}^\infty \(q(T, t)V(\M+\e_i, T - t)\)  \notag\\
=& P_i + \E[V(\M+\e_i, T - t) | T \geq t].  \label{equ:random_g}
\end{align}
Also, the average of $g(\P, i)$ over all $i$ is independent of $\P$:
\begin{align*}
&\frac{1}{N} \sum_{i=1}^N g(\P, i)  \\
=& \frac{1}{N} + \frac{1}{N} \sum_i^N \E[V(\M+\e_i, T - t) | T \geq t] \\
=& \E\left[\frac{1}{N} + \frac{1}{N} \sum_i^N V(\M+\e_i, T - t) | T \geq t\right] \\
=& \E\left[\frac{1}{N} + \frac{1}{N} \sum_i^N \(\frac{T-t}{N} - R(\M+\e_i, T - t)\) | T \geq t\right] \\
=& \E[\frac{T-t+1}{N} - R(\M, T - t + 1) | T \geq t]  \tag{by definition of R}\\
=& \E[V(\M, T - t + 1) | T \geq t] ,
\end{align*}
which implies
\begin{equation}\label{equ:lower_ran}
\min_{\P\in \Delta(N)} \max_{i\in [N]} g(\P, i) \geq \E[V(\M, T - t + 1) | T \geq t].
\end{equation}
On the other hand, let $\P^* = \E[\P^T | T \geq t]$, where $P^T_i = V(\M, T - t + 1) - V(\M + \e_i, T - t)$. $\P^*$ is a valid distribution 
since $\P^T$ is a distribution for any $T$. 
Also, by plugging into Eq. \eqref{equ:random_g},
\begin{align*}
g(\P^*, i) &= \E[V(\M, T - t + 1) - V(\M + \e_i, T - t) | T\geq t] \\
&\quad + \E[V(\M+\e_i, T - t) | T \geq t] \\
&= \E[V(\M, T - t + 1) | T \geq t].
\end{align*}
Therefore,
\begin{equation}\label{equ:upper_ran}
\begin{split}
\min_{\P\in \Delta(N)} \max_{i\in [N]} g(\P, i) \leq \max_{i\in [N]} g(\P^*, i) \\
= \E[V(\M, T - t + 1) | T \geq t].
\end{split}
\end{equation}
Eq. \eqref{equ:lower_ran} and \eqref{equ:upper_ran} show that $\min_{\P} \max_{i}  g(\P, i) = \E[V(\M, T - t + 1) | T \geq t]$, 
which agrees with the left hand side of Eq. \eqref{equ:minimax_ran_recur}. 
We thus prove 
$\adjustlimits \inf_{\Alg}\sup_{\Z_{1:\infty}} \E_{T\sim Q}[\Reg(L_T, \M_T)] 
= \E[V(\0, T) | T \geq 1] 
= \E_{T\sim Q}[\adjustlimits\inf_{\Alg}\sup_{\Z_{1:T}}\Reg(L_T, \M_T)], $
and $\P^*$ is the optimal strategy. 	
\end{proof}

\section{Proof of Theorem \ref{thm:lower_bound}}\label{app:lower_bound}
To prove Theorem \ref{thm:lower_bound}, 
we need to find out what $V(\0, T)$ is under the general loss space $[0,1]^2$.
Note that Theorem \ref{thm:minimax_fix} only tells us the case of the basis vector loss space.
Fortunately, it turns out that they are the same if $N=2$. 
To be more specific, we will show later in Theorem \ref{thm:loss_space}
that if $N=2$ and $\LS = [0,1]^2$, then $V(\0, T) = T/2 - R(\0, T)$, 
which can be further simplified as
\begin{align*}
V(\0, T) &= \frac{T}{2} - \frac{1}{2^T}\sum_{m=0}^{T}\binom{T}{m}\min\{m, T-m\} \\
&= \frac{T}{2^T}\binom{T-1}{\lfloor \frac{T}{2} \rfloor} .
\end{align*}
We can now prove Theorem \ref{thm:lower_bound} using this explicit scaling factor,
denoted by $S(T)$ for simplicity.

\begin{proof}[Proof of Theorem \ref{thm:lower_bound}]
Again, solving Eq. \eqref{equ:minimax_var} is equivalent to finding the value function 
$\V$ defined on each state of the game, 
similar to the functions $V$ and $\BV$ we had before. 
The difference is that $\V$ should be a function of not only the index of the current 
round $t$ and the cumulative loss vector $\M$,  
but also the cumulative loss $L$ for the learner. 
Moreover, to obtain a base case for the recursive definition,
it is convenient to first assume that $T$ is at most $T_0$, 
where $T_0$ is some fixed integer. 
Under these conditions, we define $\V^{T_0}_t(L, \M)$ recursively as:
\begin{align*}
 \V^{T_0}_{T_0}(L, \M) &\triangleq  \min_{\P \in \Delta(N)}\max_{\Z\in\LS} \frac{\Reg(L+\P\cdot\Z, \M+\Z)}{V(\0,T_0)},  \\
 \V^{T_0}_t(L, \M) &\triangleq \min_{\P \in \Delta(N)}\max_{\Z\in\LS} \max\bigg\{\frac{\Reg(L+\P\cdot\Z, \M+\Z)}{V(\0,t)} , \\
 &\quad\quad \V^{T_0}_{t+1}(L+\P\cdot\Z, \M+\Z) \bigg\},
\end{align*}
which is the scaled regret starting from round $t$ with cumulative loss $L$ 
for the learner and $\M$ for the actions, assuming both the learner and 
the adversary will play optimally from this round on.
The value of the game $\V$ is now $\lim\limits_{T_0\ra +\infty}\V^{T_0}_1(0, \0)$.

To simplify this value function, we will need three facts. 
First, the base case can be related to $V(\M, 1)$:
\begin{align}
&\V^{T_0}_{T_0}(L, \M) \notag \\
=&  \min_{\P \in \Delta(N)}\max_{\Z\in\LS} \frac{\Reg(L+\P\cdot\Z, \M+\Z)}{V(\0, T_0)}  \notag\\
=& \left. \(L + \min_{\P \in \Delta(N)}\max_{\Z\in\LS} \Reg(\P\cdot\Z, \M+\Z)\) \middle /V(\0, T_0) \right. \notag\\
=& \left. \(L + \min_{\P \in \Delta(N)}\max_{\Z\in\LS} \P\cdot\Z + V(\M+\Z, 0)\) \middle /V(\0, T_0) \right. \notag\\
=& \frac{L + V(\M, 1)}{V(\0, T_0)} .\notag 
\end{align}

Second, for any $L$ and $\M$, one can inductively show that
\begin{equation}\label{equ:move_loss}
\V^{T_0}_t(L, \M) = \V^{T_0}_t(L-R(\M,0), \M'),
\end{equation}
where $M'_i = M_i-R(\M,0)$. 
(We omit the details since it is straightforward.)

Third, when $\M=\0$, by symmetry, 
one has with $\P_u=(\frac{1}{N},\ldots,\frac{1}{N})$ 
\begin{align}
&\V^{T_0}_t(L, \0) \notag \\
=& \max_{\Z\in\LS} \max\left\{\frac{\Reg(L+\P_u\cdot\Z, \Z)}{V(\0,t)} ,\  \V^{T_0}_{t+1}(L+\P_u\cdot\Z, \Z) \right\} \notag \\
\geq&  \max\left\{\frac{L+\frac{1}{N}}{V(\0,t)} ,\  \V^{T_0}_{t+1}(L+\frac{1}{N}, \e_1) \right\}. \label{equ:symmetry}
\end{align}

Now we can make use of the condition $N=2$ to lower bound $\V$. 
The key point is to consider a restricted adversary who can only place one unit 
more loss on one of the action than the other, if not stopping the game.
Clearly the value of this restricted game serves as a lower bound of $\V$.
Specifically, consider the value of $\V^{T_0}_t(L, \e_1)$ for $t \leq T_0 - 2$: 
\begin{align*}
&\V^{T_0}_t(L, \e_1) \\
\geq& \min_{p \in [0,1]} \max\left\{ \frac{\Reg(L+p, 2\e_1)}{S(t)} ,\   
 \V^{T_0}_{t+1}(L+1-p, \e_1+\e_2) \right\}   \tag{restricted adversary}\\
=& \min_{p \in [0,1]} \max\left\{ \frac{L+p}{S(t)} ,\   
 \V^{T_0}_{t+1}(L-p, \0) \right\}  \tag{by Eq. \eqref{equ:move_loss}} \\
\geq& \min_{p \in [0,1]} \max\left\{ \frac{L+p}{S(t)} ,\ \frac{L+1/2-p}{S(t+1)},\ 
\V^{T_0}_{t+2}(L+\frac{1}{2}-p, \e_1) \right\}  \tag{by Eq. \eqref{equ:symmetry}} \\ 
\geq& \min_{p \in \R} \max\left\{ \frac{L+p}{S(t)} ,\ 
\V^{T_0}_{t+2}(L+\frac{1}{2}-p, \e_1) \right\}  
\end{align*}

Therefore, if we assume $T_0$ is even without loss of generality and 
define function $G^{T_0}_t(L)$ recursively as:
\begin{align*}
 G^{T_0}_{T_0}(L) &\triangleq \V^{T_0}_{T_0}(L, \e_1) 
 = \frac{L+V(\e_1, 1)}{S(T_0)} = \frac{L}{S(T_0)} \\
 G^{T_0}_t(L) &\triangleq \min_{p \in \R} \max\left\{ \frac{L+p}{S(t)} ,\ 
 G^{T_0}_{t+2}(L+\frac{1}{2}-p) \right\} ,
\end{align*}
then it is clear that $\V^{T_0}_t(L, \e_1) \geq G^{T_0}_t(L)$, 
and thus by \eqref{equ:symmetry}, 
$$ \V^{T_0}_1(0,\0) 
\geq \max\{1, \V^{T_0}_2(\frac{1}{2}, \e_1)\}  
\geq \max\{1, G^{T_0}_2(\frac{1}{2})\}. $$

It remains to compute $G^{T_0}_2(\frac{1}{2})$. 
By some elementary computations and the fact that for two linear functions 
$h_1(p)$ and $h_2(p)$ of different signs of slopes, 
$\min_p\max\{h_1(p), h_2(p)\} = h_1(p^*)$ where $p^*$ is such that
$h_1(p^*) = h_2(p^*)$, one can inductively prove that for $t = 2, 4,\ldots,T_0$, 
$$ G^{T_0}_t(L) = \frac{2^{\frac{T_0-t}{2}}(L+\frac{1}{2})-\frac{1}{2}}{S(T_0) + 
\sum\limits_{k=1}^{(T_0-t)/2} \( 2^{k-1}S(T_0-2k)\) } .$$
Plugging $S(t) = \frac{t}{2^t}\binom{t-1}{\lfloor t/2 \rfloor}$ and letting $T_0 \ra\infty$, 
we arrive at
\begin{align*}
&\lim_{T_0\ra\infty} G^{T_0}_2(1/2)  \\
=& \lim_{T_0\ra\infty} \( \sum_{k=1}^{T_0/2-1} \(2^{k-T_0/2}S(T_0-2k)\) \) ^{-1}  \\
=& \lim_{T_0\ra\infty} \( \sum_{k=1}^{T_0/2-1} \(\frac{S(2k)}{2^k}\) \) ^{-1}  \\
=& \( \sum_{j=0}^{\infty} \frac{j}{8^{j}}\binom{2j}{j} \) ^{-1} . 
\end{align*}
Define $G(x) = \sum_{j=0}^{\infty}\binom{2j}{j} x^j$ and $F(x) = x \cdot G'(x)$.
Note that what we want to compute above is exactly $1/F(\frac{1}{8})$.
\citet{Lehmer85} showed that $G(x) = (1-4x)^{-1/2}$.
Therefore, $F(x) = 2x \cdot (1-4x)^{-3/2}$ and 
$$ \lim_{T_0\ra\infty} G^{T_0}_2(1/2) = 1/F(1/8) = \sqrt{2}.$$
We conclude the proof by pointing out
\begin{align*}
\V &= \lim_{T_0\ra\infty}\V^{T_0}_1(0,\0)  \\
&\geq \max\{1,\lim_{T_0\ra\infty} G^{T_0}_2(1/2)\} = \sqrt{2}.   
\end{align*}
\end{proof}

As we mentioned at the beginning of this section, 
the last thing we need to show is that the value $V(\0, T)$ is the same under the two loss spaces.
In fact, we will prove stronger results in the following theorem 
claiming that this is true only if $N=2$.
 
\begin{theorem}\label{thm:loss_space}
Let $\LS_1, \LS_2, \LS_3$ be the three loss spaces $\{\e_1, \ldots, \e_N\}, 
\{0,1\}^N$ and $[0,1]^N$ respectively, 
and $V_\LS (\0, T)$ be the value of the game $V(\0,T)$ under the loss space $\LS$.
If $N > 2$, we have for any $T$, 
$$V_{\LS_1} (\0, T) < V_{\LS_2} (\0, T) = V_{\LS_3} (\0, T). $$
However, the three values above are the same if $N=2$. 
\end{theorem}

\begin{proof}
We first inductively show that for any $\M$ and $r$,  
$V_{\LS_2} (\M, r) = V_{\LS_3} (\M, r)$ and $V_{\LS_3} (\M, r)$ is convex in $\M$. 
For the base case $r = 0$,
by definition, $V_{\LS_2} (\M, 0) = V_{\LS_3} (\M, 0) = -\min_i M_i$. 
Also, for any two loss vectors $\M$ and $\M'$, and $\lambda \in [0,1]$, 
\begin{align*}
&V_{\LS_3} (\lambda\M + (1-\lambda)\M', 0) \\
=& -\min_i \(\lambda M_i + (1-\lambda)M'_i\) \\
\leq&  -\min_i \(\lambda M_i\)  - \min_i \( (1-\lambda)M'_i\) \\
=& \lambda V_{\LS_3} (\M, 0) + (1-\lambda) V_{\LS_3} (\M', 0),
\end{align*}
showing $V_{\LS_3} (\M, 0)$ is convex in $\M$. 
For $r > 0$, 
$$V_{\LS_3} (\M, r) = \min_{\P\in \Delta(N)} \max_{\Z\in\LS_3}  \(\P\cdot\Z + V_{\LS_3}(\M + \Z, r - 1)\). $$ 
Notice that $\P\cdot\Z + V_{\LS_3}(\M + \Z, r - 1)$ is equal to $\P\cdot\Z + V_{\LS_2}(\M + \Z, r - 1)$ and is convex in $\Z$ by induction. 
Therefore the maximum is always achieved at one of the corner points of $\LS_3$, 
which is in $\LS_2$.  In other words,
\begin{align*}
V_{\LS_3} (\M, r) &= \min_{\P\in \Delta(N)} \max_{\Z\in\LS_2}  
\(\P\cdot\Z + V_{\LS_2}(\M + \Z, r - 1)\) \\
&= V_{\LS_2} (\M, r).
\end{align*}
On the other hand, by introducing a distribution $Q$ over all the elements in $\LS_2$, 
we have
\begin{align*}
&V_{\LS_3} (\M, r)  \\
=& \min_{\P\in \Delta(N)} \max_{\Z\in\LS_2}  \(\P\cdot\Z + V_{\LS_3}(\M + \Z, r - 1)\) \\
=& \min_{\P\in \Delta(N)} \max_{Q}  \E_{\Z\sim Q} \left[\P\cdot\Z + V_{\LS_3}(\M + \Z, r - 1)\right] \\
=& \max_{Q} \min_{\P\in \Delta(N)} \E_{\Z\sim Q} \left[\P\cdot\Z + V_{\LS_3}(\M + \Z, r - 1)\right]  \\
=& \max_{Q} \( \E_{\Z\sim Q} V_{\LS_3}(\M + \Z, r - 1) + \min_{\P\in \Delta(N)}  \P\cdot\E_{\Z\sim Q}[\Z] \)
\end{align*}
where we switch the min and max by Corollary 37.3.2 of \citet{Rockafellar70}. 
Note that the last expression is the maximum over a family of linear combinations of convex functions in $\M$, 
which is still a convex function in $\M$, completing the induction step. 
To conclude, $V_{\LS_2} (\0, T) = V_{\LS_3} (\0, T)$ for any $N$ and $T$. \\

We next prove if $N=2, V_{\LS_1} (\0, T) = V_{\LS_2} (\0, T)$. 
Again, we inductively prove $V_{\LS_1} (\M, r) = V_{\LS_2} (\M, r)$ 
for any $\M$ and $r$.  The base case is clear. For $r>0$, let $P^*_i = V_{\LS_1} (\M, r) - V_{\LS_1} (\M+\e_i, r-1) \ (i=1,2)$. By induction, 
\begin{align*}
&V_{\LS_2} (\M, r)  \\
=& \min_{\P\in \Delta(2)} \max_{\Z\in\LS_2}  \(\P\cdot\Z + V_{\LS_1}(\M + \Z, r - 1)\) \\
\leq& \max_{Z_1,Z_2 \in \{0,1\}}  \(P^*_1Z_1+P^*_2Z_2 + V_{\LS_1}(\M + (Z_1,Z_2), r - 1)\) \\
=& \max \{ V_{\LS_1}(\M, r - 1),  1+V_{\LS_1}(\M+(1,1), r - 1), \\
&\quad\quad V_{\LS_1}(\M, r) \}   \\
=& \max\left\{ V_{\LS_1}(\M, r - 1),  V_{\LS_1}(\M, r) \right\} 
\tag{by Property \ref{pro:additive} in Proposition \ref{pro:minimax_fix}} \\
=&  V_{\LS_1}(\M, r). 
\tag{by Property \ref{pro:monotonicity_r} in Proposition \ref{pro:minimax_fix}}
\end{align*}
However, it is clear that $V_{\LS_2} (\M, r) \geq V_{\LS_1}(\M, r)$. 
Therefore, $V_{\LS_1} (\M, r) = V_{\LS_2} (\M, r)$. \\

Finally, to prove $V_{\LS_1} (\0, T) < V_{\LS_2} (\0, T)$ for $N>2$, we inductively prove 
$V_{\LS_1} ((T-r)\e_1, r) < V_{\LS_2} ((T-r)\e_1, r) $ for $r = 1, \ldots, T$. For the base case $r=1$,  $V_{\LS_1} ((T-1)\e_1, 1) = 1/N - R((T-1)\e_1, 1) = 1/N$, while 
\begin{align*}
&V_{\LS_2} ((T-1)\e_1, 1) \\
=& \min_{\P\in \Delta(N)} \max_{\Z\in \LS_2}  \(\P\cdot\Z + V_{\LS_2}((T-1)\e_1 + \Z, 0)\) \\
\geq& \min_{\P\in \Delta(N)} \max_{i\in [N]}  \(1-P_i + V_{\LS_2}((T-1)\e_1 + \1- \e_i, 0)\)  \\
=& \min_{\P\in \Delta(N)} \max\left\{-P_1, 1-P_2, \ldots, 1-P_N \right\}  .
\end{align*}
We claim that the value of the last minimax expression above, 
denoted by $v$, is $(N-2)/(N-1)$,
which is strictly greater than $1/N$ if $N > 2$ and thus proves the base case.  
To show that, notice that for any $\P \in \Delta(N)$, 
there must exist $i\in \{2, \ldots, N\}$ such that $P_i \leq 1/(N-1)$ and 
$$\max\left\{-P_1, 1-P_2, \ldots, 1-P_N \right\} \geq  
1-P_i \geq \frac{N-2}{N-1},$$
showing $v \geq (N-2)/(N-1)$. 
On the other hand, the equality is realized by the distribution 
$\P^* = (0, \frac{1}{N-1}, \ldots, \frac{1}{N-1})$.

For $r > 1$, we have
\begin{align*}
&V_{\LS_2} ((T-r)\e_1, r) \\
=& \min_{\P\in \Delta(N)} \max_{\Z\in \LS_2}  \(\P\cdot\Z + V_{\LS_2}((T-r)\e_1 + \Z, r-1)\) \\
\geq& \min_{\P\in \Delta(N)} \max_{i \in [N]}  \(P_i + V_{\LS_2}((T-r)\e_1 + \e_i, r-1)\) \\
\geq& \min_{\P\in \Delta(N)} \frac{1}{N}\sum_{i=1}^N  \(P_i + V_{\LS_2}((T-r)\e_1 + \e_i, r-1)\) \\
=& \frac{1}{N}  + \frac{1}{N}\sum_{i=1}^N  V_{\LS_2}((T-r)\e_1 + \e_i, r-1) \\
>& \frac{1}{N}  + \frac{1}{N}\sum_{i=1}^N  V_{\LS_1}((T-r)\e_1 + \e_i, r-1) \\
=& V_{\LS_1}((T-r)\e_1, r).
\end{align*}
Here, the last strict inequality holds because for $i=1$, $V_{\LS_2}((T-r+1)\e_1, r-1) > V_{\LS_1}((T-r+1)\e_1, r-1)$ by induction; for $i \not= 1$, it is trivial that $V_{\LS_2}((T-r)\e_1 + \e_i, r-1) \geq V_{\LS_1}((T-r)\e_1 + \e_i, r-1)$. Therefore, we complete the induction step and thus prove $V_{\LS_1} (\0, T) < V_{\LS_2} (\0, T)$. 
\end{proof}

\section{Proof of Theorem \ref{thm:adaptive_upper}}
The proof (and the one of Theorem \ref{thm:adaptive_MW}) relies heavily on a common technique to approximate a sum using an integral,
which we state without proof as the following claim.
\begin{claim}\label{clm:int}
Let $f(x)$ be a non-increasing nonnegative function defined on $\R_+$. Then the following inequalities hold for any integer $0 < j \leq k $.
$$ \int_{j}^{k+1}f(x) \ dx \leq \sum_{i = j}^k f(i) \leq \int_{j-1}^k f(x) \ dx$$
\end{claim}

\begin{proof}[Proof of Theorem \ref{thm:adaptive_upper}]
By Theorem \ref{thm:adaptive_gen_upper}, 
it suffices to upper bound $\BV_1(\0)$ and $\sum_{t=1}^{T_s} q_t \BV_{t+1}(\0)$. 
Let $S_t = \sum_{t'=t}^\infty 1/t'^d$. 
By applying Claim \ref{clm:int} multiple times,  we have 
\begin{equation}\label{equ:bound_S}
\frac{1}{S_t} \leq \(\int_t^\infty \frac{dx}{x^d}\)^{-1} = t^{d-1}(d-1) ;
\end{equation}
$$ q_t = \frac{1}{S_t \cdot t^d} \leq \frac{d-1}{t}; $$
\begin{align*}
&\BV_1(\0) = \E[V(\0, T) | T \geq 1] \\
\leq& \frac{c_N}{S_1} \sum_{T=1}^\infty \frac{1}{T^{d-\frac{1}{2}}}  
\tag{by Theorem \ref{thm:minimax_fix}}\\
\leq& \frac{c_N}{S_1} \(1 + \int_1^\infty \frac{dx}{x^{d-\frac{1}{2}}}\) \tag{by Claim \ref{clm:int}} \\
=& \frac{c_N(d-\frac{1}{2})}{S_1(d-\frac{3}{2})} \\
\leq& \frac{c_N(d-1)(d-\frac{1}{2})}{d-\frac{3}{2}}= O(1) .\tag{by Eq. \eqref{equ:bound_S}}
\end{align*}
For 
\begin{align*}
\BV_{t+1}(\0) &= \E[V(\0, T-t) | T \geq t+1] \\
&= \frac{c_N}{S_{t+1}}\sum_{k=1}^\infty \frac{\sqrt{k}}{(t+k)^d},
\end{align*}
Claim~\ref{clm:int} does not readily apply since the function 
$g(k) = \sqrt{k}/(t+k)^d$ is increasing on $[0, t/(2d-1)]$ and
then decreasing on $[t/(2d-1), \infty)$.
However, we can still apply the claim to these two parts separately.
Let $x_0 = \lfloor t/(2d-1)\rfloor$ and $x_1 = \lceil t /(2d-1)\rceil$.
For simplicity, assume $1 \leq x_0 < x_1$ and $g(x_0) \leq g(x_1)$ 
(other cases hold similarly). Then we have
\begin{align*}
\BV_{t+1}(\0) &= \frac{c_N}{S_{t+1}} \(g(x_1) + \sum_{k=1}^{x_0} g(k) 
+ \sum_{k=x_1+1}^\infty g(k)  \)\\
&\leq \frac{c_N}{S_{t+1}} \(g(x_1) + \int_{0}^{x_1} g(x) dx
+ \int_{x_1}^{\infty} g(x)dx \) \\
&= \frac{c_N}{S_{t+1}}\(g(x_1) + \frac{\Gamma(d-\frac{3}{2})}{2\Gamma(d)}\cdot\frac{\sqrt{\pi}}{t^{d-\frac{3}{2}}}  \)\\
&\leq (d-1)c_N\sqrt{\pi}\cdot\frac{\Gamma(d-\frac{3}{2})}{2\Gamma(d)}
\cdot \sqrt{t} + o(\sqrt{t}).
\end{align*}
So finally we have
\begin{align*}
&\sum_{t=1}^{T_s} q_t \BV_{t+1}(\0)  \\
\leq&  (d-1)^2c_N\sqrt{\pi}\cdot\frac{\Gamma(d-\frac{3}{2})}{2\Gamma(d)}
\sum_{t=1}^{T_s} \( \frac{1}{\sqrt{t}} + o(\frac{1}{\sqrt{t}})\)   \\
\leq& \frac{\Gamma(d-\frac{3}{2})}{\Gamma(d)}(d-1)^2c_N\sqrt{\pi T_s} 
+ o(\sqrt{T_s}), 
\end{align*}
which proves the theorem.
\end{proof}

\section{Proof of Theorem \ref{thm:ball_game}}
\begin{proof}
Let $\Phi_t^T = \sqrt{\|\W_{t-1}\|^2 + (T-t+1)}$ be the potential function
for this setting.
The key property of the minimax algorithm Eq. \eqref{equ:ball_game_minimax}
shown by \citet{AbernethyBaRaTe08} is the following:
$$ \x_t^T\cdot \w_t \leq  \Phi_t^T - \Phi_{t+1}^T. $$
Based on this property, the loss of our algorithm after $T_s$ rounds is
\begin{align*}
\sum_{t=1}^{T_s} \E[\x_t^T | T \geq t] \cdot \w_t
&= \sum_{t=1}^{T_s} \E[\x_t^T\cdot \w_t | T \geq t] \\
&\leq \sum_{t=1}^{T_s} \E[\Phi_t^T - \Phi_{t+1}^T | T \geq t] .
\end{align*}
Now define $U_t = \E[\Phi_t^T | T \geq t]$ and $q_t = \Pr[ T < t+1 | T \geq t]$.
By the fact that $f_{T \geq t}(t') = (1-q_t) f_{T \geq t+1}(t')$ for any $t' \geq t+1$, 
where $f_{T \geq t}$ and $f_{T \geq t+1}$ are conditional density functions,
we have
\begin{align*}
&\sum_{t=1}^{T_s} \E[\x_t^T | T \geq t] \cdot \w_t  \\
\leq\;& \sum_{t=1}^{T_s} \(U_t - \E[\Phi_{t+1}^T | T \geq t]\) \\
=\;& \sum_{t=1}^{T_s} \(U_t - \int_t^{t+1}\Phi_{t+1}^T f_{T\geq t}(T) dT - (1-q_t)U_{t+1}\) \\
\leq\;&  \sum_{t=1}^{T_s} \(U_t - \Phi_{t+1}^t\int_t^{t+1} f_{T\geq t}(T) dT - (1-q_t)U_{t+1}\)
\tag{$\because\Phi_{t+1}^T$ increases in $T$} \\
=\;& \sum_{t=1}^{T_s} \(U_t - U_{t+1} + q_t(U_{t+1} - \|\W_{t-1}\| )\) 
\tag{$\because \Phi_{t+1}^t = \|\W_{t-1}\|$} \\
=\;& U_1 - U_{T_s+1} + \\
&\sum_{t=1}^{T_s} q_t
\E\left[\sqrt{\|\W_{t-1}\|^2+(T-t)}- \|\W_{t-1}\| \;|\; T\geq t +1\right]  \\
\leq\;& U_1 - U_{T_s+1} + \sum_{t=1}^{T_s} q_t \E\left[\sqrt{T-t} \;|\; T \geq t+1 \right].
\tag{$\because \sqrt{a+b} - \sqrt{a} \leq \sqrt{b}$}
\end{align*}
Note that $U_{T_s+1} \geq \|\W_T\|$, and thus it remains to plug in the distribution
and compute $U_1$ and $\sum_{t=1}^{T_s} q_t \E[\sqrt{T-t} \;|\; T \geq t+1 ]$,
which is almost the same process as what we did in the proof of Theorem \ref{thm:adaptive_upper}
if one realizes $ q_t \leq (d-1)/t $ also holds here.
In a word, the regret can be bounded by
$$\frac{\Gamma(d-\frac{3}{2})}{\Gamma(d)}(d-1)^2 \sqrt{\pi T_s} 
+ o(\sqrt{T_s}) ,$$
which is $\pi\sqrt{T_s} + o(\sqrt{T_s})$ if $d=2$. 
The explicit form in Eq. \eqref{equ:ball_game_adaptive} comes from a direct calculation.
\end{proof}

\section{Proof of Lemma~\ref{lem:FPL} and Theorem \ref{thm:FPL}}
\begin{proof}[Proof of Lemma~\ref{lem:FPL}]
The results follow by a direct calculation. The conditional distribution
of $\bxi_t$ given $T$ is 
$\frac{1}{\Delta_T^N} \1\{\bxi \in [0, \Delta_T]^N\}.$
Let $S_t = \int_t^\infty 1/T^d dT = \((d-1)t^{d-1}\)^{-1}.$ 
The marginal distribution for $\bxi$ that has negative coordinates is clearly $0$.
Otherwise, with $\bar{t} = \max\{t, \frac{\|\bxi\|_\infty^2}{bN}\}$ one has
\begin{align*}
f_t(\bxi) &= \frac{1}{S_t} \int_t^\infty \frac{1}{T^d \Delta_T^N}
\1\{\bxi \in [0, \Delta_T]^N\} dT \\
&= \frac{1}{S_t} \int_{\bar{t}}^\infty \frac{1}{T^d \Delta_T^N} dT \\
&= \frac{(d-1)t^{d-1}}{(\sqrt{bN})^N} \int_{\bar{t}}^\infty \frac{1}{T^{d+N/2}} dT \\
&= \frac{d-1}{d-1+N/2} \Delta_t^{-N}\min\left\{1,  
\(\frac{\Delta_t}{\|\bxi\|_\infty}\)^{2d-2+N}\right\}.
\end{align*}
\end{proof}

\begin{proof}[Proof of Theorem \ref{thm:FPL}]
Applying Theorem 4.2 of \citet{CesabianchiLu06}, the pseudo-regret 
of the FPL algorithm is bounded by
\begin{align*}
&\E[\max_i \xi_{T_s, i}] + \sum_{t=1}^{T_s} \E[\max_i (\xi_{t-1, i} - \xi_{t,i}) ] \\
& + \sum_{t=1}^{T_s} \int_{\R^N} F_t(\bxi)(f_t(\bxi) - f_t(\bxi - \Z_t)) d\bxi,
\end{align*}
where we define $\bxi_0 = \0$ and $F_t(\bxi) = Z_{t, I_\bxi}$ with 
$I_\bxi \in \arg\min_i (M_{t-1, i} + \xi_i)$.
Now the key observation is that the pseudo-regret remains the same
if we replace random variables $\bxi_1,\ldots,\bxi_{T_s}$ with 
$\bxi_1',\ldots,\bxi_{T_s}'$ as long as $\bxi_t$ and $\bxi_t'$ have the
same marginal distribution for any $t$. 
Specifically, we can let $\bxi_{T_s}' = \bxi_{T_S}$, and for $1 < t \leq T_s$,
let $\bxi_{t-1}' = \bxi_{t}'$ with probability $S_{t} / S_{t-1} = (1-1/t)^{d-1}$ 
(recall $S_t = \int_t^\infty 1/T^d dT$),
or with $1-S_t/S_{t-1}$ probability be obtained 
by first drawing $T \in [t-1, t]$ according to density
$f(T) \propto 1/T^d$, and then drawing a point uniformly in $[0, \Delta_T]^N$.
It is clear that $\bxi_t$ and $\bxi_t'$ have the same marginal distribution.
So the pseudo-regret can be in fact bounded by three terms:
\begin{align*}
A &= \E[\max_i \xi_{T_s, i}], \\
B &= \sum_{t=1}^{T_s} \E[\max_i (\xi_{t-1, i}' - \xi_{t,i}') ], \\
C &= \sum_{t=1}^{T_s} \int_{\R^N} F_t(\bxi)(f_t(\bxi) - f_t(\bxi - \Z_t)) d\bxi.
\end{align*}
$A$ can be further bounded by
$$
\frac{1}{S_{T_s}}\int_{T_s}^\infty \frac{\Delta_T}{T^d}  dT 
= \frac{d-1}{d-3/2}\sqrt{bT_sN}.
$$
For $B$, by construction of $\bxi_t'$, we have
\begin{align*}
B &\leq \sum_{t=2}^{T_s} \( \frac{\Delta_t}{S_{t-1}}\int_{t-1}^t \frac{dT}{T^d} 
+ \frac{S_t}{S_{t-1}} \cdot 0 \) \\
&=\sum_{t=2}^{T_s}  \frac{\Delta_t}{t^{d-1}} \(t^{d-1} - (t-1)^{d-1}\)\\
&\leq \sum_{t=2}^{T_s}  \frac{\Delta_t}{t^{d-1}} \cdot (d-1)t^{d-2} \tag{by convexity}\\
&\leq 2(d-1)\sqrt{bT_sN}.
\end{align*}
For $C$, let $H = \{\bxi: f_t(\bxi) > f_t(\bxi-\Z_t)\}$. 
Since $0 \leq F_t(\bxi) \leq 1$, we have
$ C \leq \sum_{t=1}^{T_s} \int_{H} f_t(\bxi) d\bxi .$
Now observe that when $\min_i \xi_i \geq 0$, $f_t(\bxi)$ 
is non-increasing in each $\xi_i$. So the only possibility that
$f_t(\bxi) > f_t(\bxi-\Z_t) $ holds is when there exists an $i$ such that
$\xi_i$ is strictly smaller than $Z_{t,i}$. That is 
$$  H = \{\bxi:  \min_i \xi_i \geq 0 \text{\;and\;} \exists i, s.t. \; \xi_i < Z_{t,i}\} $$
So we have
\begin{align*}
C &\leq \sum_{t=1}^{T_s} \frac{1}{S_t} \int_t^\infty 
\frac{dT}{T^d}\int_H \frac{\1\{\bxi \in [0, \Delta_T]^N\}}{\Delta_T^N} d\bxi \\
&\leq \sum_{t=1}^{T_s} \frac{1}{S_t} \int_t^\infty 
\frac{N}{T^d} \frac{Z_{t,i} \Delta_T^{N-1}}{\Delta_T^N}  dT\\ 
&\leq \frac{d-1}{d-1/2}\sqrt{\frac{N}{b}} \sum_{t=1}^{T_s} \frac{1}{\sqrt{t}} \\ 
&\leq  \frac{2(d-1)}{\sqrt{b}(d-1/2)} \sqrt{T_sN}.
\end{align*}
Combining $A, B$ and $C$ proves the theorem.
\end{proof}

\section{Proof of Theorem \ref{thm:adaptive_MW}}
\begin{proof}
We will first show that 
\begin{equation}\label{equ:MW_gen_bound}
\begin{split}
\Reg(L_{T_s}, \M_{T_s}) \leq 
\underbrace{(\ln N) \cdot \E\left[\frac{1}{\eta_T} | T \geq T_S + 1 \right]}_A   \\+  
\underbrace{\frac{1}{8}\sum_{t=1}^{T_s}\E[\eta_T | T \geq t]}_B .
\end{split}
\end{equation}
Let  $\Phi_t^T = \frac{1}{\eta_T}\ln\(\sum_{i=1}^N \exp(-\eta_T M_{t-1, i})\)$. 
The key point of the proof for the non-adaptive version of the exponential weights 
algorithm is to use $\Phi_t^T$ as a ``potential'' function, 
and bound the change in potential before and after a single round \cite{CesabianchiLu06}. Specifically, they showed that
$$ \P^T_t \cdot \Z_t \leq \frac{\eta_T}{8} + \Phi_t^T - \Phi_{t+1}^T .$$
We also base our proof on this inequality. 
The total loss of the learner after $T_s$ rounds is 
\begin{align*}
L_{T_s} &= \sum_{t=1}^{T_s}  \E[\P^T_t  | T \geq t] \cdot \Z_t= \sum_{t=1}^{T_s}  \E[\P^T_t \cdot \Z_t | T \geq t] \\
&\leq B + \sum_{t=1}^{T_s}\E[ \Phi_t^T - \Phi_{t+1}^T | T \geq t] .
\end{align*}
Define $U_t = \E[ \Phi_t^T | T\geq t]$. We do the following transformation:
\begin{align*}
&\E[\Phi_t^T - \Phi_{t+1}^T  | T \geq t] \\
=& U_t - E_T[\Phi_{t+1}^T  | T \geq t] \\
=& U_t - q_t \Phi_{t+1}^t - (1-q_t)U_{t+1} \\
=& U_t - U_{t+1} + q_t(U_{t+1} - \Phi_{t+1}^t)  \\
=& U_t - U_{t+1} + q_t\cdot\E[\Phi_{t+1}^T - \Phi_{t+1}^t | T \geq t + 1] \\
=& U_t - U_{t+1} + q_t\cdot\E[F_{T, t}(\M_t) | T \geq t + 1] ,
\end{align*}
where we define $$F_{T, t}(\M) = \frac{\ln\(\sum_i \exp(-\eta_T M_i)\)}{\eta_T} - 
\frac{\ln\(\sum_i \exp(-\eta_t M_i)\)}{\eta_t} .$$
A key observation is 
\begin{equation}\label{equ:key_ob}
\max \limits_{\M \in \R^N_+ \atop \eta_T<\eta_t} F_{T,t}(\M) = \frac{\ln N}{\eta_T} -\frac{\ln N}{\eta_t},
\end{equation}
which can be verified by a standard derivative analysis that we omit. 
(An alternative approach using KL-divergence can be found in Chapter 2.5 of \citealp{Bubeck11}.)

We further define another potential function $\bar{\Phi}_t^T=(\ln N)/\eta_T $
and also $\bar{U}_t = \E[\bar{\Phi}_t^T | T \geq t]$. 
Note that the new potential $\bar{\Phi}_t^T$ has no dependence on $t$
and thus $\bar{\Phi}_t^T = \bar{\Phi}_{t'}^T$ for any $t, t'$.
We now have
\begin{align} 
&\sum_{t=1}^{T_s} \E[\Phi_t^T - \Phi_{t+1}^T  | T \geq t] \notag \\
=\;& \sum_{t=1}^{T_s} \(U_t - U_{t+1} + q_t\cdot\E[\Phi_{t+1}^T 
- \Phi_{t+1}^t | T \geq t + 1] \)  \notag \\
=\;& \underbrace{U_1 - U_{T_s+1}  + \sum_{t=1}^{T_s}  \( q_t  \cdot 
\E[\Phi_{t+1}^T - \Phi_{t+1}^t | T \geq t + 1] \)}_C  \label{equ:transformation}\\
\leq\;& U_1 - U_{T_s+1} + \sum_{t=1}^{T_s} \(q_t \cdot \E[\frac{\ln N}{\eta_T} 
- \frac{\ln N}{\eta_t} | T \geq t + 1] \)   \tag{by Eq. \eqref{equ:key_ob}}\\
=\;& \underbrace{\bar{U}_1 - \bar{U}_{T_s+1} + \sum_{t=1}^{T_s} \(q_t  \cdot \E[\bar{\Phi}_{t+1}^T 
- \bar{\Phi}_{t+1}^t | T \geq t + 1] \)}_D  \notag\\
&+ \bar{U}_{T_s+1} - U_{T_s+1}  
\tag{$\because U_1 = \bar{U}_1$}  .
\end{align}
Notice that $D$ has the exact same form as $C$ except for a different definition
of the potential, and also Eq. \eqref{equ:transformation} is an equality.
Therefore, by a reverse transformation, we have
\begin{align*}
&\sum_{t=1}^{T_s} \E[\Phi_t^T - \Phi_{t+1}^T  | T \geq t]  \\
=& \sum_{t=1}^{T_s} \E[\bar{\Phi}_t^T - \bar{\Phi}_{t+1}^T  | T \geq t]
+ \bar{U}_{T_s+1} - U_{T_s+1}   \\
=& \bar{U}_{T_s+1} - U_{T_s+1}  \tag{$\because \bar{\Phi}_t^T = \bar{\Phi}_{t+1}^T$}
\end{align*}
$\bar{U}_{T_s+1}$ is exactly $A$ in Eq. \eqref{equ:MW_gen_bound}, 
and $U_{T_s+1}$ can be related to the loss of the best action:
\begin{align*}
U_{T_s+1} &= 
\E\left[ \frac{1}{\eta_T}\ln\sum_{i=1}^N \exp(-\eta_T M_{T_s, i}) \;|\; T \geq T_s + 1\right] \\
&\geq \E\left[ \frac{1}{\eta_T}\ln\exp(-\eta_T R(M_{T_s}, 0)) \;|\; T \geq T_s + 1\right] \\
&= -R(M_{T_s}, 0) .
\end{align*}
The regret is therefore 
\begin{align*}
\Reg(L_{T_s}, \M_{T_s}) &= L_{T_S} - R(M_{T_s}, 0)  \\
&\leq A + B - U_{T_s+1} - R(M_{T_s}, 0)  \\
&\leq A + B, 
\end{align*}
proving Eq. \eqref{equ:MW_gen_bound}.

The rest of the proof is merely to plug in the distribution and $\eta_T = \sqrt{(b\ln N)/T}$,
and upper bound Eq. \eqref{equ:MW_gen_bound} using Claim \ref{clm:int}.
Adopting the notation $S_t =\sum_{t'=t}^\infty 1/t'^d$ and the result
of Eq. \eqref{equ:bound_S} in the proof of Theorem \ref{thm:adaptive_upper}, we have
\begin{align*}
A &= \frac{\sqrt{\ln N}}{S_{T_s+1}\sqrt{b}} \sum_{T=T_s+1}^\infty \frac{1}{T^{d-1/2}} \\
&\leq \frac{(d-1)\sqrt{\ln N}}{\sqrt{b}}(T_s+1)^{d-1} \cdot \\
&\quad\(\int_{T_s+1}^\infty \frac{dx}{x^{d-1/2}} +\frac{1}{(T_s+1)^{d-1/2}}  \)  \\
&= \frac{d-1}{(d-3/2)\sqrt{b}} \sqrt{T_s\ln N}  + o(\sqrt{T_s\ln N}) ;
\end{align*}
\begin{align*}
B &= 
\frac{\sqrt{b\ln N}}{8}\sum_{t=1}^{T_s} \frac{1}{S_t}\sum_{T=t}^\infty \frac{1}{T^{d+1/2}} \\
&\leq \frac{(d-1)\sqrt{b\ln N}}{8}\sum_{t=1}^{T_s} t^{d-1} \(
\int_{t}^\infty \frac{dx}{x^{d+1/2}} + \frac{1}{t^{d+1/2}}\) \\
&\leq \frac{(d-1)\sqrt{b\ln N}}{8}\sum_{t=1}^{T_s} \( 
\frac{1}{(d-1/2)\sqrt{t}} + \frac{1}{t^{d+3/2}}\) \\
&\leq \frac{\sqrt{b}(d-1)}{4(d-1/2)}\sqrt{T_s\ln N}  + o(\sqrt{T_s\ln N}).
\end{align*}
Combining the bounds above for $A$ and $B$ proves the theorem. 
\end{proof}

\section{Proof of Theorem \ref{thm:first_order}}
\begin{proof}
The main idea resembles the one of Theorem \ref{thm:adaptive_MW},
but the details are much more technical.
Let us first define several notations:
\begin{align*}
S_t &\triangleq \int_{m_t}^\infty \frac{dm}{m^d} = \frac{1}{(d-1)m_t^{d-1}} , \\
q_t &\triangleq \Pr[m < m_t | m \geq m_{t-1}] = \frac{1}{S_{t-1}}
\int_{m_{t-1}}^{m_t} \frac{dm}{m^d} \\ &= 1 - \(\frac{m_{t-1}}{m_t}\)^{d-1} ,\\
Y_t^m &\triangleq \sum_{i=1}^N\exp(-\eta_m M_{t-1,i}) ,\\
\Phi_t^m &\triangleq \(1+\frac{1}{\eta_m}\)\ln Y_t^m   ,\quad
U_t \triangleq  \E[\Phi_t^m | m \geq m_{t-1}].
\end{align*}
The proof starts from the following property of the exponential weights algorithm
\cite{CesabianchiLu06}:
\begin{align*}
\P_t^m \cdot \Z_t &\leq 
\frac{1}{1-e^{-\eta_m}}\(\ln Y_t^m - \ln Y_{t+1}^m\) \\
&\leq \Phi_t^m - \Phi_{t+1}^m. \tag{$\because \eta_m \geq \ln(1+\eta_m)$}
\end{align*}

By the fact that $f_{m \geq m_{t-1}}(m') = (1-q_t) f_{m\geq m_t}(m')$ for any $m' \geq m_t$, 
where $f_{m \geq m_{t-1}}$ and $f_{m \geq m_t}$ are conditional density functions,
the loss of the learner after $T_s$ rounds $L_{T_s}$ is 
\begin{align*}
& \sum_{t=1}^{T_s} \E[\P_t^m \cdot \Z_t | m \geq m_{t-1}] \\
\leq& \sum_{t=1}^{T_s} \E[\Phi_t^m - \Phi_{t+1}^m | m \geq m_{t-1}] \\
=& \sum_{t=1}^{T_s} \(U_t - \int_{m_{t-1}}^{m_t} \Phi_{t+1}^m f_{m\geq m_{t-1}}(m) dm
+ (1-q_t) U_{t+1} \) \\
\leq& \sum_{t=1}^{T_s} \(U_t - \Phi_{t+1}^{m_{t-1}} \int_{m_{t-1}}^{m_t}  f_{m\geq m_{t-1}}(m) dm + (1-q_t) U_{t+1} \) \\
=& U_1 - U_{T_s+1} + \sum_{t=1}^{T_s} q_t (U_{t+1} -  \Phi_{t+1}^{m_{t-1}}),
\end{align*}
Here the last inequality holds because $\Phi_t^m$ is increasing in $m$.
To show this, we consider the following
\begin{align*}
&\(1+\frac{1}{\eta}\)\ln \sum_{i=1}^N \exp(-\eta a_i) \\
=& \(1+\frac{1}{\eta}\)\(-\eta a_1
+ \ln \sum_{i=1}^N \exp(-\eta(a_i-a_1)) \) \\
=& -(\eta+1) a_1 +  \(1+\frac{1}{\eta}\)\ln \sum_{i=1}^N \exp(-\eta(a_i-a_1)),
\end{align*}
where $\eta, a_1, \ldots, a_N$ are positive numbers. 
Since $\ln \sum_i\exp(-\eta(a_i-a_1)) \geq 0$, 
the expression above is decreasing in $\eta$, 
which along with the fact that $\eta_m$ decreases in $m$
shows that $\Phi_t^m$ increases in $m$.

We now compute $U_1$ and $U_{T_s+1}$:
\begin{align*}
U_1 &= \E[(1+\sqrt{m/\ln N})\ln N \;|\; m \geq 1] \\
&= \ln N + \frac{d-1}{d-3/2}\sqrt{\ln N} \\
U_{T_s+1} &= \E\left[(1+1/\eta_m) \ln\sum_i \exp(-\eta_m M_{T_s,i}) 
\;|\; m \geq m_{T_s} \right ] \\
&\geq \E[(1+1/\eta_m) (-\eta_m m^*) \;|\; m \geq m_{T_s}] \\
&= -m^* \(1 + \E[\eta_m \;|\; m \geq m_{T_s}] \) \\
&= -m^* \(1 + \frac{d-1}{d-1/2} \sqrt{\frac{\ln N}{m_{T_s}}} \) \\
&\geq -m^* - \frac{d-1}{d-1/2} \sqrt{m^* \ln N} \tag{$\because m_{T_s} = m^* + 1$}
\end{align*}

For $U_{t+1} -  \Phi_{t+1}^{m_{t-1}} = \E[\Phi_{t+1}^m - 
\Phi_{t+1}^{m_{t-1}} \;|\; m \geq m_t]$,
we first upper bound the part inside the expectation:
\begin{align*}
&\Phi_{t+1}^m - \Phi_{t+1}^{m_{t-1}} \\
=\;& \(\frac{\ln Y_{t+1}^m}{\eta_m} - \frac{\ln Y_{t+1}^{m_{t-1}}}{\eta_{m_{t-1}}}\) + (\eta_{m_{t-1}} - \eta_m) \min_i M_{t,i} \\
&+ \ln \frac{\sum e^{-\eta_m (M_{t,i} - \min_i M_{t,i})}}
{\sum e^{-\eta_{m_{t-1}} (M_{t,i} - \min_i M_{t,i})}}.
\end{align*}
The first term above is at most 
$\(\frac{1}{\eta_m} - \frac{1}{\eta_{m_{t-1}}}\)\ln N
= \sqrt{\ln N}(\sqrt{m} - \sqrt{m_{t-1}}) $ by Eq. \eqref{equ:key_ob}.
The second term is at most $\sqrt{\ln N}(\frac{1}{\sqrt{m_{t-1}}}-\frac{1}{\sqrt{m}})m_{t-1}$ 
since $\min_i M_{t,i} = m_t - 1 \leq m_{t-1}$,
and the last term is at most $\ln N$ since the 
numerator is at most $N$ while the denominator is at least $1$.
Therefore, we have
\begin{align*}
&U_{t+1} -  \Phi_{t+1}^{m_{t-1}}  \\
\leq& \ln N + \sqrt{\ln N}\cdot \E[\sqrt{m} - \frac{m_{t-1}}{\sqrt{m}} \;|\; m \geq m_t] \\
=& \ln N + \sqrt{\ln N} \( \frac{d-1}{d-3/2}\sqrt{m_t} - 
\frac{d-1}{d-1/2}\frac{m_{t-1}}{\sqrt{m_t}} \) \\
\leq& \ln N + \sqrt{\ln N} \( \frac{d-1}{d-3/2}\sqrt{m_t} - 
\frac{d-1}{d-1/2}\frac{m_{t}-1}{\sqrt{m_t}} \) \\
=&  \ln N + \frac{(d-1)\sqrt{m_t\ln N}}{(d-3/2)(d-1/2)}  + 
\frac{d-1}{d-1/2}\sqrt{\frac{\ln N}{m_t}}.
\end{align*}
It remains to compute $\sum_{t=1}^{T_s} q_t (U_{t+1} -  \Phi_{t+1}^{m_{t-1}})$,
which, using the above, can be done by computing
$A = \sum_{t=1}^{T_s} q_t$, 
$B = \sum_{t=1}^{T_s} q_t \sqrt{m_t}$ and
$C = \sum_{t=1}^{T_s} q_t / \sqrt{m_t}$.
By inequality $1 - x \leq -\ln x$ for any $x >0$, we have
\begin{align*}
A &= \sum_{t=1}^{T_s} \(1 - \(\frac{m_{t-1}}{m_t}\)^{d-1}\) \\
&\leq -(d-1) \sum_{t=1}^{T_s} \(\ln m_{t-1} - \ln m_t\) \\
&= (d-1)\ln (m^*+1) .
\end{align*}
For $B$, we first show $q_t \sqrt{m_t} \leq 2(d-1)(\sqrt{m_t} - \sqrt{m_{t-1}})$,
which is equivalent to 
$$ \frac{q_t \sqrt{m_t}}{\sqrt{m_t} - \sqrt{m_{t-1}}} 
= \frac{\(\frac{m_t}{m_{t-1}}\)^{d-1} - 1}
{\(\frac{m_t}{m_{t-1}}\)^{d-1} - \(\frac{m_t}{m_{t-1}}\)^{d-3/2}} 
\leq 2(d-1)
$$
if $m_t \neq m_{t-1}$ (it is trivial otherwise). 
Define $h(x) = (x^{d-1}-1)/(x^{d-1} - x^{d-3/2})$ for $x \in [1,2]$ 
(note that $m_t/m_{t-1}$ is within this interval).
One can verify that $h'(x) < 0$ and thus 
$h(x) \leq \lim_{x\rightarrow 1} h(x) = 2(d-1)$.
So we prove $q_t \sqrt{m_t} \leq 2(d-1)(\sqrt{m_t} - \sqrt{m_{t-1}})$ and
\begin{align*}
B &\leq 2(d-1) \sum_{t=1}^{T_s} (\sqrt{m_t} - \sqrt{m_{t-1}}) \\
&= 2(d-1) (\sqrt{m_{T_s}} - 1) \leq 2(d-1)\sqrt{m^*}.
\end{align*}
A simple comparison of $B$ and $C$ shows $C = o(\sqrt{m^*})$.
We finally conclude the proof by combining all we have
\begin{align*}
&\Reg(L_{T_s}, \M_{T_s}) \\
\leq\;& U_1 - U_{T_s+1} + \sum_{t=1}^{T_s} q_t (U_{t+1} -  \Phi_{t+1}^{m_{t-1}}) - m^* \\
=\;& (1+(d-1)\ln(m^*+1))\ln N  \\ 
&+ \(\frac{d-1}{d-1/2} + \frac{2(d-1)^2}{(d-3/2)(d-1/2)}\)\sqrt{m^*\ln N} \\
&+ o(\sqrt{m^*\ln N})\\
=\;& \frac{3(d-7/6)(d-1)}{(d-3/2)(d-1/2)}\sqrt{m^* \ln N} \\
&+ (1+(d-1)\ln(m^*+1))\ln N +  o(\sqrt{m^*\ln N}). 
\end{align*}
\end{proof}

\section{Examples}
The first example shows that the results stated in Theorem~\ref{thm:minimax_ran}
can not generalize to other loss spaces.
\begin{example}\label{exm:minimax_ran}
Consider the following Hedge setting: $N=3, \LS = \{\1-\e_1, \1-\e_2, \1-\e_3\}$
where $\1=(1,1,1)$.
Suppose the adversary picked $\1-\e_1$ and $\1-\e_2$ for the first two rounds
and we are now on round $t=3$ with $\M_2=(1,1,2)$. Also the conditional distribution
of the horizon given $T\geq 3$ is $\Pr[T=3]=\Pr[T=4]=1/2$. 
Let $\P^*$ be the minimax strategy for this round and $\P^T$ be the minimax
strategy assuming the horizon to be $T$. Then $\P^* \neq \E[\P^T | T \geq 3]$, and
also
\begin{equation}\label{equ:exm_fix_and_ran}
\begin{split}
\adjustlimits \inf_{\Alg}\sup_{\Z_{3:\infty}} \E[\Reg(L_T, \M_T) | T \geq 3] \\ 
\neq \E[\adjustlimits\inf_{\Alg}\sup_{\Z_{3:T}}\Reg(L_T, \M_T) | T \geq 3]. 
\end{split}
\end{equation}
\end{example}
\begin{proof}
Recall the $V$ function we had in Section~\ref{sec:fh}. 
Ignoring the loss for the learner for the first two rounds 
(which is the same for both sides of Eq.~\eqref{equ:exm_fix_and_ran}),
we point out that the right hand side of Eq.~\eqref{equ:exm_fix_and_ran}
is essentially 
$$  \frac{1}{2}V(\M_2, 1) + \frac{1}{2}V(\M_2, 2), $$
and the left hand side, denoted by $V'$, is
$$ \min_\P\max_\Z (\P\cdot \Z + \frac{1}{2}V(\M_2+\Z, 0) + \frac{1}{2}V(\M_2+\Z,1)).$$
Also $\P^*$ and $\P^T$ are the distributions that realize the minimum
in the definition of $V'$ and $V(\M_2, T-2)$ respectively.
Below we show the values of these quantities without giving full details:
\begin{align*}
V(\M_2, 1) &= \min_\P\max_i\{1-P_i+V(\M_2+\1-\e_i, 0)\} \\
&= \min_\P\max\{-P_1, -P_2, -P_3-1\} \\
&= -1/2,
\end{align*}
with $\P^3 = (1/2,1/2,0)$;
\begin{align*}
V(\M_2, 2) &= \min_\P\max_i\{1-P_i+V(\M_2+\1-\e_i, 1)\} \\
&= \min_\P\max\{-P_1, -P_2, -P_3-1/3\} \\
&= -4/9,
\end{align*}
with $\P^4 = (4/9,4/9,1/9)$;
\begin{align*}
V' &= \min_\P\max_i \Big(1-P_i + \frac{1}{2}V(\M_2+\1-\e_i, 0) \\
&\quad\quad + \frac{1}{2}V(\M_2+\1-\e_i,1)\Big)  \\
&= \min_\P\max\{-P_1, -P_2, -P_3-2/3\} \\
&= -1/2,
\end{align*}
with $\P^* = (1/2,1/2,0).$
We thus conclude that 
$$ \E[\P^T | T \geq 3] = (17/36,17/36,1/18) \neq \P^* $$
and
$$ \E[V(\M_2, T-2) | T \geq 3] = -17/36 \neq V'.$$
\end{proof}

The next two examples show that the idea of ``treating the current round as 
the last round'' does not work for minimax algorithms.
\begin{example}\label{exm:last_round1}
Consider the following Hedge setting: $N=2, \LS=[0,1]^2$ and the horizon $T$
is a even number. Suppose on round $t$, the learner chooses $\P_t$ using the
minimax algorithm assuming horizon $T=t$. Then the adversary can make the regret
after $T$ rounds to be $T/4$ by choosing $\e_1$ and $\e_2$ alternatively.
\end{example}
\begin{proof}
As shown in Theorem~\ref{thm:loss_space}, when $N=2$, the minimax algorithm
with $\LS=[0,1]^2$ is the same as the one with $\LS=\{\e_1,\e_2\}$, which we
already know from Theorem~\ref{thm:minimax_fix}. 
If the learner treats the current round as the last round, then $P_{t,1}$ is
\begin{align*}
&V(\M_{t-1},  1) - V(\M_{t-1} + \e_1, 0) \\
=& \frac{1}{2}\big(1 + \min\{M_{t-1,1}+1, M_{t-1,2}\} \\
&- \min\{M_{t-1,1}, M_{t-1,2}+1\}\big).
\end{align*}
Hence, for any round $t$ where $t$ is odd, we have 
$\M_{t-1} = (\frac{t-1}{2}, \frac{t-1}{2})$ and thus $P_{t,1} = P_{t,2} = 1/2$
and the learner suffers loss $1/2$.
For any round $t$ where $t$ is even, we have 
$\M_{t-1} = (\frac{t}{2}, \frac{t}{2}-1)$ and thus $P_{t,1} = 0, P_{t,2} = 1$
and the learner suffers loss $1$ since the adversary will choose $\e_2$ for this round.
Finally, at the end of $T$ rounds, the loss of the best action is clearly $T/2$.
So the regret would be $3T/ 4 - T/2 = T/4$.
\end{proof}

\begin{example}\label{exm:last_round2}
Consider the online linear optimization problem described in 
Section~\ref{subsec:ball_game}. If horizon $T$ is even and the learner
predicts using the minimax algorithm Eq~\eqref{equ:ball_game_minimax} 
with $T$ replaced with $t$. Then the adversary can make the regret to
be $\sqrt{2}T/4$ after $T$ rounds by choosing $\e_1$ and $-\e_1$ alternatively.
\end{example}
\begin{proof}
For any round $t$ where $t$ is odd, we have $\W_{t-1} = \0$ and thus $\x_t=\0$.
So the loss for this round is $0$.
For any round $t$ where $t$ is even, we have $\W_{t-1} = \e_1$ and 
thus $\x_t= -\frac{\sqrt{2}}{2}\e_1$. So the loss for this round is $\sqrt{2}/2$
since the adversary will pick $-\e_1$.
At the end of $T$ rounds, since $\W_T = \0$, the regret will simply be $\sqrt{2}T/4$.
\end{proof}

\end{document}